\documentclass[reprint, amsmath, amssymb, aps, onecolumn, superscriptaddress, floatfix]{revtex4-2}
\pdfoutput=1

\usepackage{microtype}
\usepackage{graphicx}
\usepackage{booktabs} 
\usepackage{enumitem}
\usepackage{xcolor}
\usepackage{color}
\usepackage{float}
\usepackage[style=base]{caption}
\usepackage{subcaption}
\usepackage{tikz}
\usepackage{tikz-cd}
\usepackage{hyperref}
\usepackage{pgfplots}
\usepgfplotslibrary{groupplots}
\pgfplotsset{compat=1.18}

\usepackage{amsmath}
\usepackage{amssymb}
\usepackage{mathtools}
\usepackage{amsthm}
\usepackage{physics}
\usepackage{bbm}
\usepackage{dsfont}
\usepackage{mathrsfs}

\usepackage[capitalize,noabbrev]{cleveref}

\theoremstyle{plain}
\newtheorem{theorem}{Theorem}[section]

\newtheorem{corollary}[theorem]{Corollary}
\theoremstyle{definition}

\theoremstyle{remark}

\newcommand{\pout}[1]{\ensuremath{P_{\mathrm{out}}\qty(#1)}}

\newcommand{\Zout}{\ensuremath{\mathcal{Z}_{\mathrm{out}}}}
\newcommand{\Zoutstar}{\ensuremath{\mathcal{Z}_{\mathrm{out}^\star}}}
\newcommand{\Zw}{\ensuremath{\mathcal{Z}_{\mathrm{w}}}}
\newcommand{\Zwstar}{\ensuremath{\mathcal{Z}_{\mathrm{w}^\star}}}
\newcommand{\fout}{\ensuremath{f_{\mathrm{out}}}}
\newcommand{\foutstar}{\ensuremath{f_{\mathrm{out}^\star}}}
\newcommand{\fw}{\ensuremath{f_{\mathrm{w}}}}
\newcommand{\fwstar}{\ensuremath{f_{\mathrm{w}^\star}}}

\newcommand{\EEb}[2]{\ensuremath{\mathbb{E}_{#1}\qty[#2]}}
\newcommand{\RR}{\ensuremath{\mathbb{R}}}

\newcommand{\din}{\ensuremath{\Delta_{\text{\tiny{IN}}}}}
\newcommand{\dout}{\ensuremath{\Delta_{\text{\tiny{OUT}}}}}
\newcommand{\deff}{\ensuremath{\Delta_{\text{eff}}}}

\newcommand{\zetain}{\ensuremath{\zeta_{\text{\tiny{IN}}}}}
\newcommand{\zetaout}{\ensuremath{\zeta_{\text{\tiny{OUT}}}}}

\newcommand{\chiin}{\ensuremath{\chi_{\text{\tiny{IN}}}}}
\newcommand{\chiout}{\ensuremath{\chi_{\text{\tiny{OUT}}}}}

\newcommand{\lambdaopt}{\ensuremath{\lambda_{\mathrm{opt}}}}

\newcommand{\zetainzero}{\ensuremath{\zeta_{\text{\tiny{IN}},0}}}
\newcommand{\zetaoutzero}{\ensuremath{\zeta_{\text{\tiny{OUT}},0}}}

\newcommand{\ystar}{\ensuremath{y^\star}}

\newcommand{\losshuber}[1]{{\ell^{\rm Huber}_{#1}}}

\newcommand{\wstar}{\ensuremath{\boldsymbol{w}^\star}}
\newcommand{\what}{\ensuremath{\hat{\boldsymbol{w}}}}

\newcommand{\Egen}{\ensuremath{E_{\mathrm{gen}}}}

\newcommand{\Eest}{\ensuremath{E_{\mathrm{estim}}}}
\newcommand{\excessEgen}{\ensuremath{E_{\mathrm{gen}}^{\mathrm{excess}}}}

\newcommand{\ntest}{\ensuremath{n_{\mathrm{t}}}}

\begin{abstract}
We study robust linear regression in high-dimension, when both the dimension $d$ and the number of data points $n$ diverge with a fixed ratio $\alpha=n/d$, and study a data model that includes outliers. We provide exact asymptotics for the performances of the empirical risk minimisation (ERM) using $\ell_2$-regularised $\ell_2$, $\ell_1$, and Huber losses, which are the standard approach to such problems. We focus on two metrics for the performance: the generalisation error to similar datasets with outliers, and the estimation error of the original, unpolluted function. Our results are compared with the information theoretic Bayes-optimal estimation bound. For the generalization error, we find that optimally-regularised ERM is asymptotically consistent in the large sample complexity limit if one perform a simple calibration, and compute the rates of convergence. 
For the estimation error however, we show that due to a norm calibration mismatch, the consistency of the estimator requires an oracle estimate of the optimal norm, or the presence of a cross-validation set not corrupted by the outliers. We examine in detail how performance depends on the loss function and on the degree of outlier corruption in the training set and identify a region of parameters where the optimal performance of the Huber loss is identical to that of the $\ell_2$ loss, offering insights into the use cases of different loss functions.
\end{abstract}

\begin{document}

    \title{Asymptotic Characterisation of Robust Empirical Risk Minimisation Performance \\ in the Presence of Outliers}
    
    \author{Matteo~Vilucchio}
    \affiliation{Information, Learning and Physics Laboratory, EPFL, 1015 Lausanne, Switzerland}
    \author{Emanuele~Troiani}
    \affiliation{Statistical Physics of Computation Laboratory, EPFL, 1015 Lausanne, Switzerland}
    \author{Vittorio~Erba}
    \affiliation{Statistical Physics of Computation Laboratory, EPFL, 1015 Lausanne, Switzerland}
    \author{Florent~Krzakala}
    \affiliation{Information, Learning and Physics Laboratory, EPFL, 1015 Lausanne, Switzerland}
    
    \maketitle

    \section{Introduction}

    The problem of robust regression --- \textit{i.e.} linear regression with spurious outliers in the training set --- has been extensively studied in the classical low-dimensional setting \cite{Huber2011}. 
    More recently, the high-dimensional setting has also been considered mainly concerning the characterisation of the asymptotic performance of robust estimators --- see Section~\ref{sec:related-works} for a review of related works.
    In this paper, we provide new theoretical results on the problem in the high-dimensional regime where the dimensionality of the problem $d$ is comparable to the number of training samples $n$ available ($n, d \to \infty$ with fixed ratio $\alpha = n/d$), investigating in particular the phenomenology arising due to the presence of outliers.

    We build upon the recent progresses in high-dimensional statistics that characterise the asymptotic performance of the Bayes Optimal (BO) and the Empirical Risk Minimisation (ERM) estimators for general data models learned with generalised linear models.
    In particular, we consider a probabilistic model of training datasets --- first introduced in \cite{Box} --- that features a tunable presence of outliers and we investigate the high-dimensional setting performances of some standard regression estimators. We focus our attention on empirical risk minimisers under the $\ell_2$-regularised $\ell_2$, $\ell_1$ and Huber~losses \cite{huber1964}, and study how they asymptotically perform in various outlier settings (e.g. small/large outlier percentage, small/large outlier variance). 
    We focus on two metrics for the performance, the \textit{generalisation error} to similar datasets with outliers, and the \textit{estimation error} of the original, noiseless function, and we compare to information theoretic Bayes-optimal bounds.
    
    For the probabilistic model of outliers we consider, eq.~(\ref{eq.datamodel}), we provide closed-form expressions for the high-dimensional limit of the generalisation error of ERM for the $\ell_2$, $\ell_1$ and Huber~losses, and for the BO estimator (for which we compute exact rates at large sample complexity $\alpha$). 
    \textbf{Our main results are:}
    \begin{itemize}[noitemsep,wide=0pt]
        \item \textbf{Optimally-regularised ERM is not always consistent.}
        Based on the value of the parameters characterising the outliers' distribution, optimally-regularised ERM estimation can fail to achieve BO performance for large sample complexity $\alpha$, i.e. fail to achieve consistency. 
        This happens for both the generalisation error and for the estimation error, in different regions of the parameter space.
        Notably, for the generalisation error there is a regime in which optimally-regularised ERM with $\ell_1$ loss or the Huber loss (with un-optimised scale parameter) is not consistent, while optimally-regularised ERM with $\ell_2$ and optimally-scaled Huber loss is consistent.
        This demonstrates the superiority of optimally-scaled Huber loss over the other losses.
        \item \textbf{Calibration of ERM minimiser norm restores consistency.}
        In the aforementioned cases in which ERM is not consistent, the inconsistency is caused exclusively by a mismatch between the norm of the ERM minimiser and that of the ground-truth. We  show that consistency for the generalization error, and optimal rates, can  be recovered by cross-validating on a test set to find the optimal norm.
        \item \textbf{Non-consistency of the estimation error.}
        Even when the cross-validated estimator is consistent and achieves optimal rates for the generalization error, it may remain inconsistent for the estimation error on the original, unpolluted function, which is arguably a more interesting objective. While this, too, can be calibrated, such a calibration requires either having access to an oracle knowledge of the BO norm, or to a clean, unpolluted dataset to cross-validate for the optimal norm. In the absence of such an ideal hold-out dataset, the Huber and $\ell_1$ methods do {\it not} lead to consistent estimates of the objective functions, and robust estimation is biased.
        \item \textbf{Dependency of the generalisation error on the parameters.}
        We study the dependency of the generalisation error on the amount and strength of the outliers in the training dataset.
        We find that there is an extended region of parameters where the performance of the Huber~loss is the same as the simpler $\ell_2$ function, offering insights into the use cases of different loss functions.
    \end{itemize}

    In Section~\ref{sec:related-works} we discuss related works.
    In Section~\ref{sec:data-noise-model} we introduce the data model we consider, and the noise model accounting for outliers.
    In Section~\ref{sec:bo-and-erm} we provide the analytical characterisation for the estimation and generalisation errors for the ERM and BO estimators.
    In Section~\ref{sec:consistency-errors} we discuss the large sample complexity limit, and prove our results concerning the consistency of ERM estimators.
    In Section~\ref{sec:explore} we discuss the dependency of the generalisation error on the parameters characterising the outliers' distribution. We provide more details and proofs figures in the Supplemental Materials. The code to reproduce the figures can be found as a package available at \url{https://github.com/IdePHICS/RobustRegression}.

    \section{Related works}\label{sec:related-works}

    The problem of regression in the presence of outliers has been studied extensively, starting with Huber's seminal works \cite{huber1964, Huber1965} where he first introduced M-estimators as a robust alternative to classical estimators and secondly extended the concept of robustness to hypothesis testing. A comprehensive framework for understanding and defining robustness in statistical procedures was then given by \cite{hampel1968contributions, hampel1971}. 
    
    Several algorithms have been proposed to approach problems with outliers --- see the review \cite{Huber2011}. 
    In most cases, these are variants of standard linear regression based, for example, on custom loss functions \cite{huber1964}, on hard-thresholding \cite{Bhatia}, on some sparsity assumption \cite{Suggala} or on the sum of squares method \cite{Raghavendra}.

    The choice of the outlier model greatly affects both the theoretical analyses and practical algorithms. 
    The double Gaussian noise model that we consider has been first introduced in \cite{Box}. A similar model that considers outliers as a two-way process is studied in \cite{BergerBerliner1986}. Recently this model has been reconsidered by \cite{Bellec_20}, where they provide bounds on the generalisation error in the case of Huber-loss regression. 
    Other models have also been proposed; see, for example, \cite{Bean,  Tsakonas, Yuen, Dalalyan, Gagnon}. 

    There have been numerous works analysing the generalisation error for M-estimation and robust regression in high dimensions.
    In \cite{elkaraui, Bean, Donoho}, they consider high-dimensional M-estimation for a model of data with exponentially-distributed outlier noise. They provide results on exact asymptotics and study the trade-off between using $\ell_2$ and $\ell_1$~loss. 
    In \cite{hastie}, they prove bounds for the precise asymptotics of the generalisation error of $\ell_2$ regularised $\ell_2$~loss, while in \cite{Thrampoulidis}, they prove asymptotics for M-estimation using Gordon min-max techniques.
    In \cite{Liu_M_estimation}, they consider a sparse model under an almost-high-dimensional limit $n \sim \log(d)$. 
    In \cite{Huang20}, a comparison of the precise asymptotics of $\ell_1$-regularised $\ell_2$, $\ell_1$ and Huber~loss is presented, with a focus on the interplay between the sparsity and the sample complexity of the problem. 
    In \cite{Bellec_21}, they present results on the asymptotic normality of M-estimators.

    The characterisation of the generalisation error for ERM that we use and specialise to our specific data model eq.~(\ref{eq.datamodel}) has been proven in \cite{aubin2020generalisation} and more recently in greater generality in \cite{Loureiro_2022}. The characterisation of the Bayes optimal performance has been rigorously established in \cite{Barbier2019}. We are using these results as the basis of our investigations in this paper.  We stress that all these works do not consider model outliers, thus providing  no direct insight on robust regression.
    
    \section{Data and Noise Model}\label{sec:data-noise-model}

    We will consider a linear regression task in $\mathbb{R}^d$, \textit{i.e.} we will fit a linear model --- the ``student'' --- 
    $\hat y(\boldsymbol{x}; \boldsymbol{w}) = \boldsymbol{x} \cdot \boldsymbol{w} / \sqrt{d}$ to a dataset $\mathcal{D} = \qty{(\boldsymbol{x}_i, y_i)}_{i=1}^{n}$ made of $n$ pairs of data-points $\boldsymbol{x}_i \in \mathbb{R}^d$ and labels $y_i \in \mathbb{R}$ sampled independently from a joint probability distribution $P(\boldsymbol{x}, y)$.
    
    We will consider a probabilistic generative model $P(\boldsymbol{x}, y) = P_{\rm out}(y | \boldsymbol{x} \cdot \wstar / \sqrt{d} ) P_{\rm input}(\boldsymbol{x})$. In this model, the dataset labels are generated by corrupting the output of a ``teacher'' linear model with weights $\boldsymbol{w}^*$ through an output channel $P_{\rm out}$. We will consider the simple Gaussian case where the teacher vector is distributed as $\wstar\sim\mathcal{N}(0, \mathbbm{1}_d)$, and the samples $\boldsymbol{x}$ have Gaussian distribution $P_{\rm input} = \mathcal{N}(0, \mathbbm{1}_d)$. 
    Assuming Gaussian data may seem limiting, but a long line of work has proven to various degrees that real data often behave as Gaussian data due to an underlying mechanism called “Gaussian universality”, see for example \cite{HuLu23}. Notice also that correlated feature may be addressed under the setting of \cite{Loureiro_2022}. We do not do that, as we found an interesting phenomenology already for uncorrelated features.
    
    In order to model datasets with outliers, we consider the following generative process for the labels
    \begin{equation} \label{eq.datamodel}
        P_{\rm out}(y | \ystar )
        = \begin{cases}
            \ystar  + z\sqrt{\din} & \text{with prob. } 1 - \epsilon \\
            \beta \ystar + z\sqrt{\dout} & \text{with prob. } \epsilon \\ 
        \end{cases}
    \end{equation}
    where $\ystar = \boldsymbol{x} \cdot \wstar / \sqrt{d}$ is the clean teacher-generated label, $z$ is a standard Gaussian noise, $\din$ is the variance of the non-outlying training points, $\dout$ is the variance of the outlying training points, $\epsilon$ is the fraction of outliers and 
    $\beta \geq 0$ is a parameter controlling the outliers' norm. 
    In the following, we will consider often the special case of $\beta = 0$, but all of our results can be extended for any value $\beta \in \mathbb{R}_+$ (see Appendix~\ref{SM:thm-ERM}).

    We will consider two different metrics to analyse the performance of the estimators. The first metric is the generalisation error of a model with weights $\boldsymbol{w}$, defined as 
    \begin{equation}\label{eq:gen-error}
        \Egen(\boldsymbol{w}) = \EEb{ (\boldsymbol{x}_{\text{new}}, y_{\text{new}}) }{ \qty(y_{\text{new}} - \frac{\boldsymbol{x}_{\text{new}}\cdot \boldsymbol{w}}{\sqrt{d}} )^2} \,, 
    \end{equation}
    where the average is over a sample-label pair $(\boldsymbol{x}_{\text{new}}, y_{\text{new}})$ generated with the same distribution as the training set.
    The second metric we will consider is the estimation error, which is the mean-square error between the teacher and student weights, defined as
    \begin{equation}\label{eq:estimation-error}
        \Eest(\boldsymbol{w}) = \frac{1}{d}\norm{\wstar - \boldsymbol{w}}_2^2 \, ,
    \end{equation}
    We remark that, as the sample distribution $ P_{\rm input}(\boldsymbol{x})$ is Gaussian, the estimation error also coincides with the generalisation error evaluated on a clean dataset
    \begin{equation}
        \Eest(\boldsymbol{w}) = \EEb{ \boldsymbol{x} }{ \left( \ystar(\boldsymbol{x}) - \frac{\boldsymbol{x} \cdot \boldsymbol{w}}{\sqrt{d}} \right)^2} \,, 
    \end{equation}
    where the average is over the sample distribution $ P_{\rm input}(\boldsymbol{x})$, and $\ystar = \boldsymbol{x} \cdot \wstar / \sqrt{d}$ is the noiseless teacher-generated label. Thus, in this setting, cross-validation on clean datasets is equivalent to cross-validation on the mean-square error with the ground truth.
    We will be interested in characterising the behaviour of the generalisation error and the estimation error when both the dimension $d$ and the number of training samples $n$ are large, $d, n \to \infty$ with finite sample complexity $\alpha = n/d$.

    \section{Characterisation of the performance of BO and ERM estimators}\label{sec:bo-and-erm}

    \paragraph{Bayes Optimal estimator.}
    The BO estimator is defined as the estimator achieving, on average over the data distribution, the lowest value of a given performance metric.
    Define the posterior distribution as 
    \begin{equation}
        P(\boldsymbol{w} \mid \mathcal{D}) =  \frac{1}{\mathcal{Z}_{\mathcal{D}}} P(\boldsymbol{w}) P(\mathcal{D} \mid \boldsymbol{w}) \, ,
    \end{equation}
    where $\mathcal{Z}_{\mathcal{D}}$ is a normalisation factor, $\mathcal{D}$ is the training dataset, and in our case $P(\boldsymbol{w}) = \mathcal{N}(0, \mathbbm{1}_d)$.
    It is well known \cite{cover1991information} that the optimal estimator error, as defined in eq.~(\ref{eq:estimation-error}), is achieved from the posterior average 
    \begin{equation}
        \what^{\rm BO}_{\rm estim}( \mathcal{D}) = 
        \EEb{ \boldsymbol{w} | \mathcal{D} }{ \boldsymbol{w} }
    \end{equation}
    Instead, the optimal estimator for the generalisation error, as defined in eq.(\ref{eq:gen-error}) is given by (see Appendix~\ref{SM:bo-gen-error})
    \begin{equation}
        \what^{\rm BO}_{\rm gen}( \mathcal{D}) = \EEb{ \boldsymbol{w} | \mathcal{D} }{ \EEb{ \bar{\boldsymbol{x}}, \bar{y} | \boldsymbol{w} }{ \bar{y} \, \bar{\boldsymbol{x}} } } \,,
    \end{equation}
    where $\EEb{ \boldsymbol{w} | \mathcal{D} }{ \cdot }$ is the posterior average, and $\EEb{ \bar{\boldsymbol{x}}, \bar{y} | \boldsymbol{w} }{ \cdot }$ is the average over a sample-label pair $(\bar{\boldsymbol{x}}, \bar{y})$ generated with the data model eq.~(\ref{eq.datamodel}) using $\boldsymbol{w}$ as a teacher vector. 
    
    In the high-dimensional limit, one can characterise analytically the behaviour of both the BO generalisation error and the BO estimation error. This was first rigorously proven in \cite{Barbier2019}, whose results we adapt to our situation to characterise the information theoretic Bayesian bound. In \cite{Barbier2019} the authors showed that in generalized linear models, the value of the estimation error concentrates on
    \begin{equation}\label{eq:est-error-BO}
        \Eest^{\mathrm{BO}} = 1 - q_{\mathrm{b}} \,,
    \end{equation}
    while the value of the generalisation error concentrates on (here we specify their equations to the particular model of data used in the present paper)
    \begin{equation}\label{eq:gen-error-BO}
        \Egen^{\mathrm{BO}} = 1 + \epsilon (\beta^2 - 1) - q_{\mathrm{b}} (1 + \epsilon (\beta - 1))^2 + \deff \, ,
    \end{equation}
    where $\deff = (1-\epsilon)\din + \epsilon\dout$ and the order parameter $q_{\mathrm{b}}$ --- which can be interpreted as the normalised overlap between the posterior average and the teacher vector --- can be found by solving the system of two equations
    \begin{equation}\label{eq:bayes-opt-fpes}
    \begin{aligned}
        q_{\mathrm{b}} &= \frac{\hat{q}_{\mathrm b}}{1 + \hat{q}_{\mathrm b}} \, ,
        \quad 
        \hat{q}_{\mathrm{b}} =\alpha \int\mathbb{E}_{\xi}\left[\mathcal{Z}_{\mathrm{out}}\left(y, q_{\mathrm{b}}^{1 / 2} \xi, \rho_{\mathrm{w}^{\star}}-q_{\mathrm{b}}\right) 
        f_{\mathrm{out}}\left(y, q_{\mathrm{b}}^{1 / 2} \xi, \rho_{\mathrm{w}^{\star}}-q_{\mathrm{b}}\right)^{2}\right] \, dy,
    \end{aligned}
    \end{equation}
    with $\fout$ and $\Zout$ defined as 
    \begin{equation}
    \begin{split}
        \Zout(y,\omega,V) &= \mathbb{E}_{z\sim\mathcal{N}(\omega,V)}\left[\pout{y | z}\right] \, , \quad
        \fout(y,\omega,V) = \partial_\omega \log (\Zout(y,\omega,V)) \, .
    \end{split}
    \end{equation}
 
    As $\alpha \to \infty$, $q_b \to 1$, showing that for large sample complexity the BO estimation goes to zero, while the BO generalisation error does not, as a non-zero excess error remains. For this reason, in the following we will often work with the \textit{excess generalisation error}, defined as
    \begin{equation}\label{eq.excess-gen-error}
        \Egen^{\mathrm{excess}}(\boldsymbol{w}) 
        = \Egen(\boldsymbol{w}) - \Egen^{\mathrm{BO}}(\alpha = +\infty) 
        = \Egen(\boldsymbol{w}) - \epsilon (1-\epsilon) (1-\beta)^2 - \deff\,.
    \end{equation}
    
    Notice that, in practical scenarios, computing the posterior average, or more generally providing an algorithm that efficiently reaches Bayes optimal performance --- for example, in polynomial time in $n, d$ --- is challenging. In this specific case, an efficient algorithm to compute the BO estimator accurately in the considered limit on large $n,d$ and fixed ratio $n/d$ exists. It is the generalised approximate message passing (GAMP) algorithm \cite{Kabashima, Rangan, Barbier2019}, which we will use to plot numerical simulations for BO errors in the following (see Appendix~\ref{SM:GAMP} for details). 
    
    We will use the BO estimator as a baseline to compare the performance of ERM with the selected losses to the optimal one in the presence of outliers. Specifically, eq.~(\ref{eq:est-error-BO}) and eq.~(\ref{eq:gen-error-BO}) provide the information-theoretical lower bounds for the estimation error and generalisation error, respectively.
    
    \paragraph{Empirical risk minimisation.}
    Define the risk function 
    \begin{equation}\label{eq:risk-def}
        \mathcal{R}(\boldsymbol{w}) = \sum_{i=1}^n \mathcal{L}(y_i,\hat{y}_i (\boldsymbol{w})) + \lambda\: r(\boldsymbol{w}) \, ,
    \end{equation}
    where $\hat y_i(\boldsymbol{w}) = \boldsymbol{x}_i \cdot \boldsymbol{w}/\sqrt{d}$ are the labels predicted by the model on each training sample, $y_i$ are the noisy labels of the training dataset,  $\mathcal{L}$ is a convex function of $\hat{y}$, $r$ a convex function of $\boldsymbol{w}$ and $\lambda \geq 0$.
    Minimising the risk function $\mathcal{R}(\boldsymbol{w})$ provides an estimate $\what$ of the teacher vector $\wstar$, and this estimate is uniquely defined thanks to convexity of the risk.
    
    The risk function is a sum of a data-dependent term $\mathcal{L}(y, \hat y)$ and of a regularisation function $r(\boldsymbol{w})$.
    For the data-dependent loss function $\mathcal{L}$, we will consider three convex losses: $\ell_2$, $\ell_1$ and Huber~loss. 
    The  $\ell_2$~loss $\ell_2(y,\hat y) = \frac{1}{2} (y-\hat y)^2$ is the most common choice in regression and will provide a good baseline for more outlier-robust losses.
    The $\ell_1$~loss $\ell_1(y,\hat y) = \abs{y-\hat y}$ is the simplest choice for outlier-robust regression, as it penalises less poor interpolation of outliers w.r.t. the $\ell_2$~loss.
    The Huber~loss
    \begin{equation}\label{eq:huber-definition}
        \losshuber{a}(y, y^\prime) =
        \begin{cases}
            \frac{1}{2}(y- y^\prime)^2 &\text{if } \abs{y-y^\prime}<a\\ 
            a \abs{y- y^\prime}-\frac{a^2}{2} &\text{if } \abs{y-y^\prime}\geq a\\
        \end{cases}
    \end{equation}
    has been proposed as a middle-ground between $\ell_2$ and $\ell_1$ \cite{huber1964}.
    It is more resilient to outliers than the $\ell_2$~loss, while remaining smooth and convex as the $\ell_2$~loss.
    Proper tuning of the scale parameter $a$ effectively produces an $\ell_2$~loss landscape for the non-outlying data points and an $\ell_1$-like one for the outliers.
    For $a \to +\infty$ the Huber~loss reduces to the $\ell_2$~loss, while for $a \to 0$ there is a straightforward mapping to a $\ell_1$~loss (one needs to rescale the total loss by a factor $a$ to take this limit, obtaining an effective regularisation $\lambda a$).
    While our analysis could be carried on for any choice of convex regulariser $r(\boldsymbol{w})$, we will only study the $\ell_2$ regularisation $r(\boldsymbol{w}) = \frac{1}{2}\|\boldsymbol{w}\|_2^2$. 
    Notice that due to our choices of scaling, for $\lambda$ to have an effect on the norm of the learned weights in the limit $\alpha \to \infty$, one needs the scaling $\lambda = \mathcal{O}(\alpha)$.

    We will mostly consider optimally-tuned risk minimisation, meaning that we will set $\lambda$, and $a$ in the case of the Huber loss, such as to minimise either the generalisation or the estimation error.
    In practice, the former is equivalent to cross-validation on a dataset statistically identical to the training set, while the latter to cross-validation on a clean dataset, without any noise or outliers.

    We will perform numerical simulations for the errors of ERM estimators using, for the $\ell_2$ loss, the explicit solution for its minimiser. For Huber loss, we use the L-BFGS method \cite{Liu1989} initialised on a random Gaussian weight. For the $\ell_1$ loss we use its mapping onto Huber with small scale parameter, setting $a=10^{-3}$. 

    Our first main result is the analytic characterisation of the generalisation and estimation errors of the ERM estimator.
    \begin{theorem}\label{thm:main-results}
        For the ERM estimator of the risk function eq.~(\ref{eq:risk-def}) 
        with $\ell_2$ regularisation and $\lambda \geq 0$,
        under the data model defined in eq. (\ref{eq.datamodel}) with $\beta = 0$
        and in the high dimensional limit $n, d \to \infty$ with $n/d = \alpha$ fixed,
        we have that the excess generalisation error and the estimation error concentrate on 
        \begin{equation}
        \begin{aligned}
            \excessEgen \,&= 1 - \epsilon + q - 2  m  (1 - \epsilon ) \,, \\ 
            \Eest \,&= 1 + q - 2m \,,
        \end{aligned}
        \end{equation}
        where the values of $q$ and $m$ are the solutions of a system of six self-consistent equations for the unknowns $(m, q, \Sigma, \hat{m}, \hat{q}, \hat{\Sigma})$.
        Three equations are loss-independent
        \begin{equation}\label{eq:l2-regularization-fpe}
            m = \frac{\hat m}{\lambda + \hat\Sigma}\,, \quad q = \frac{\hat m^2 + \hat q}{(\lambda + \hat\Sigma)^2}\,, \quad \Sigma = \frac{1}{\lambda+\hat \Sigma}\,,
        \end{equation}
        while the other three are loss-dependent.
        For the $\ell_2$~loss
        we have that
        \begin{equation}
            \begin{aligned}
                \hat m &= \frac{\alpha}{1 + \Sigma} (1 - \epsilon) 
                \,, \quad \hat \Sigma = \frac{\alpha}{1 + \Sigma} \, , \\
                \hat q &= \frac{\alpha}{(1 + \Sigma)^2} \qty[ 1 + q + (1-\epsilon) \din + \epsilon \dout - \epsilon - 2 m \qty(1- \epsilon)] \, .
            \end{aligned}
        \end{equation}
        For the $\ell_1$~loss
        and for the Huber~loss
        we have that:
        \begin{equation}
            \begin{aligned}
                \hat m &= \frac{\alpha}{\nu} \qty[(1-\epsilon) \erf\qty(\chiin)] 
                \, , \quad
                \hat \Sigma = \frac{\alpha}{\nu }  \qty[(1-\epsilon ) \erf\qty(\chiin) + \epsilon \erf\qty(\chiout)] \, , \\
                \hat q &= \frac{\alpha}{\nu ^2} \Bigg[(1-\epsilon) \left(\zetain - \mu ^2\right) \erf\qty(\chiin) + \epsilon \left(\zetaout-\mu ^2\right) \erf\qty(\chiout)
                \\
                &\quad 
                - \mu  \sqrt{\frac{2}{\pi }} \left( (1-\epsilon ) \sqrt{\zetain} \, e^{-\chiin^2}\right. + \left.\epsilon \sqrt{\zetaout} \, e^{-\chiout^2} \right) + \mu^2\Bigg] \, , 
            \end{aligned}
        \end{equation}
        where 
        $\zetain = \din -2 m+q+1$, $\zetaout = \dout+q$, while the other constants depend on the loss.
        For $\ell_1$~loss
        \begin{equation}
            \mu^{\ell_1} = \nu^{\ell_1} = \Sigma
            \, , \quad
            \chiin^{\ell_1} = \Sigma / \sqrt{2 \zetain}
            \, , \quad
            \chiout^{\ell_1} = \Sigma / \sqrt{2 \zetaout}
            \, ,
        \end{equation}
        and for Huber~loss
        \begin{equation}
            \nu^{\text{Huber}} = \Sigma + 1
            \, , \quad
            \mu^{\text{Huber}} = a\nu^{\text{Huber}} 
            \, , \quad
            \chiin^{\text{Huber}} = \mu^{\text{Huber}} / \sqrt{2\zetain} 
            \, , \quad
            \chiout^{\text{Huber}} = \mu^{\text{Huber}} / \sqrt{2\zetaout}
            \, .
        \end{equation}
    \end{theorem}

    The previous Theorem can be proven following the proof of Theorem~1 in \cite{Loureiro_2022}, through a reduction to our case and a computation of loss-specific quantities (namely their proximal operators) that is provided in detail in Appendix~\ref{SM:thm-ERM}. 
   We also provide a similar result for $\beta > 0$.
    The parameters $m$ and $q$ appearing in the previous statement have a simple interpretation as the values around which the teacher-student and student-student overlaps concentrate in high dimension
    \begin{equation}
        m = \frac{1}{d} \wstar \cdot \what 
        \,, \quad 
        q = \frac{1}{d} \norm{\what}_2^2
        \,.
    \end{equation}
    
    In the case of Ridge~regression we have that the equations can be solved explicitly and get to an explicit form for the generalisation error. We show this in Appendix~\ref{SM:corollary-L2}.

    \section{Consistency (and lack thereof) of the generalisation \& estimation error}\label{sec:consistency-errors}

    We now study, for both generalisation and estimation error and at large sample complexity $\alpha$, whether optimally-regularised ERM estimators achieve BO performance or not, \textit{i.e.} whether ERM is a consistent estimator.

    We start by remarking that the large sample complexity asymptotics of the BO generalisation error and of the BO estimation error is 
    \begin{equation}
        \Eest^{\mathrm{BO}} = 0 
        + \mathcal{O}\left(\alpha^{-1}\right)
        \, , \quad 
        \Egen^{\mathrm{BO}} = 
        \epsilon(1-\epsilon)(1-\beta)^2 + \Delta_{\rm eff}
        + \mathcal{O}\left(\alpha^{-1}\right)
    \end{equation}
    which can be found by expanding around $q_b = 1 + c/\alpha + \mathcal{O}\left(\alpha^{-2}\right)$ in eq.~(\ref{eq:bayes-opt-fpes}).  We prove the (fast) rate $1/\alpha$ for the errors
    in Appendix~\ref{SM:rate-BO}. This rate is a typical feature of BO estimators, see for example \cite{aubin2020generalisation,haussler1991estimating}.

    \begin{figure}[t!]
        \centering
        \includegraphics[width=0.49\columnwidth]{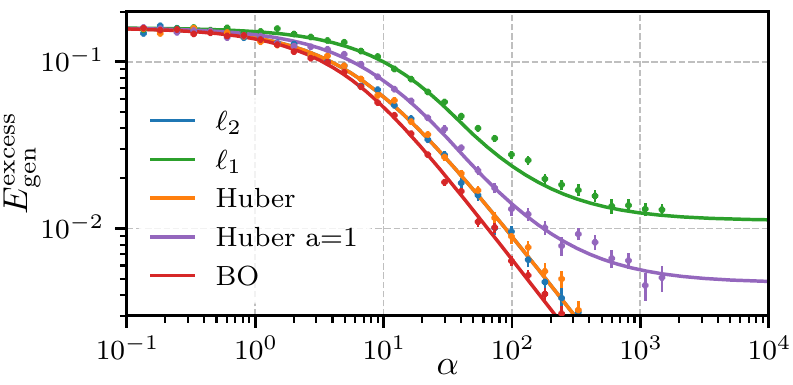}
        \hfill
        \includegraphics[width=0.49\columnwidth]{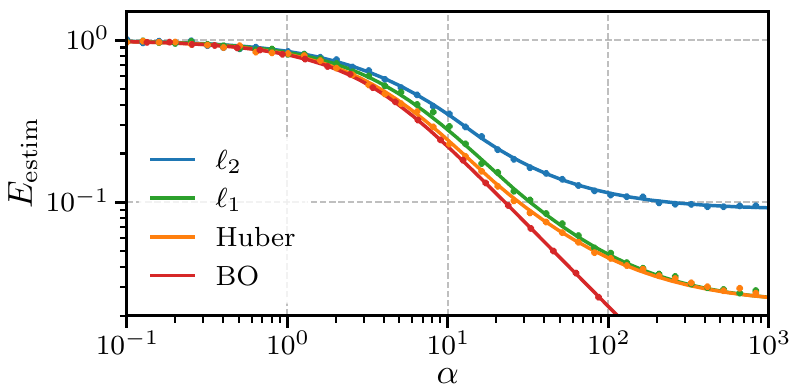}
        \hfill
        \caption{
        Generalisation and estimation errors as a function of the sample complexity $\alpha$ for $\beta = 0$, in two regimes where ERM do not always achieve BO performance for large $\alpha$.
        \textbf{(Left)} Here we plot the excess generalisation error for $\epsilon = 0.6$, $\din = 1$ and $\dout = 0.5$, and for all losses we use the optimal value of $\lambda$.
        We plot two versions of the Huber loss, one with optimal scale parameter $a$ and the other with fixed $a=1$.
        We see that only the error of the $\ell_2$ and Huber loss with optimal scale $a$ (superimposed in the plot) vanishes at large sample complexity, reaching BO performance. For fixed-scale Huber loss, and $\ell_1$ loss, the error converges to a finite value.
        The parameters used for the simulations (dots) of BO are $d=1000$ averaged over $50$ samples, while for the ERM $d=200$ averaged over $1000$ samples.
        \textbf{(Right)} Here we plot the estimation error for $\epsilon = 0.3$, $\din = 1.0$ and $\dout = 5.0$, and for all losses we use optimal $(\lambda, a)$. We see that all losses converge to a finite value of the estimation error, while the BO error goes to zero for large sample complexity. The parameters used for BO are $d=4000$ averaged over $10$ samples, while for the ERM $d=200$ aveaged over $1000$ samples.
        }
        \label{fig:figure-1}
        \vspace{-0.25cm}
    \end{figure}

    We then consider ERM.
    We find that ERM always estimates correctly the direction of the teacher weights. One could then expect that by adjusting the $\ell_2$ regularisation parameter, ERM estimators could also achieve the BO norm for the metric under consideration (generalisation/estimation error). 
    We show that this is not always possible if we restrict $\lambda \geq 0$ --- a very natural constraint given that for $\lambda < 0$ the risk function for $\ell_1$ and Huber losses becomes non-convex and unbounded from below.
    This is due to the fact that for some values of the parameters characterising the outliers' distribution, the norm of ERM estimators is shorter than the BO one, requiring \textit{a priori} a negative regularisation to push it to larger norms.    
    \begin{theorem}\label{thm:large-alpha-ERM}
        Under the same conditions of Theorem \ref{thm:main-results}, as $\alpha \to \infty$ (taken after the $n, d \to \infty$ limit), we have that
        \begin{itemize}[noitemsep,wide=0pt, topsep=0pt]
            \item the normalised scalar product between ERM estimators (for the three losses considered) with the teacher weights converges to one, regardless of the value of the regulariser $\lambda$ and of the scale parameter $a$ of the Huber loss.
            \item $\ell_2$ and optimally-scaled Huber ERM estimators achieve consistency with respect to the generalisation error for all fixed values of $\lambda > 0$ with the BO rate $1/\alpha$ for all values of the parameters of the outliers' distribution. 
            \item Optimally-regularised  $\ell_1$ and Huber ERM estimators (with the scale parameter of the Huber loss left fixed, un-optimised) achieve consistency with respect to the generalisation error with the BO rate $1/\alpha$ if and only if 
            $\dout - \din \geq (1-\beta)^2 ( 2\epsilon-1)$ for $0<\beta<1$, and $\dout - \din \leq (1-\beta)^2 ( 2\epsilon-1)$ for $\beta>1$.
            \item Optimally-regularised ERM estimators (for the three losses considered) achieve consistency with respect to the estimation error with the BO rate $1/\alpha$ if and only if $\beta > 1$.
        \end{itemize}
    \end{theorem}
    We prove Theorem \ref{thm:large-alpha-ERM} in Appendix~\ref{SM:thm-large-alpha} by expanding at large $\alpha$ the self-consistent equations presented in Theorem \ref{thm:main-results}. 

    Thus, we proved that there exists some regions of parameters in which, for both generalisation and estimation error, $\ell_2$ regularisation is not sufficient to reach BO performance. In those regimes, consistency for the generalisation error can always be recovered by cross-validating on a test set to find the best norm, given that the angle between teacher and student vanishes.
    Due to homogeneity properties of the three losses considered, the rescaling can also be performed before training on the labels, as discussed in Appendix~\ref{SM:thm-large-alpha-rescaling}.
    The estimation error can be calibrated similarly, but we stress that such a calibration requires either having access to an oracle knowledge of the BO norm, or to a clean, unpolluted dataset to cross-validate for the optimal norm. 
    In the absence of such an ideal hold-out dataset, the Huber and $\ell_1$ methods do {\it not} lead to consistent estimates of the objective functions, and robust estimation is biased.
    
    In Figure\ref{fig:figure-1} we show, for both the generalisation and estimation error, the dependency of the error on the sample complexity $\alpha$ in two regimes were some of the losses do not achieve consistency.

    It is worth noticing that consistency could be achieved by carefully allowing for negative $\ell_2$ regularisation. In the case of the $\ell_2$ loss, a negative regularisation does not immediately break the convexity of the risk function. Indeed, the risk function is convex up to $ \lambda = - (1 - \sqrt{\alpha})^2$ (see Appendix~\ref{SM:claim-L2-negative-lambda}).
    Thus, for $\ell_2$ loss consistency for the estimation error can be achieved without using an oracle norm up to (in the large sample complexity limit) $\beta > 1 - \epsilon^{-1}$, by allowing for negative lambda and retaining convexity of the risk function.
    On the other hand, for $\ell_1$ and Huber loss, even an arbitrarily small negative regularisation implies a non-convex and unbounded-from-below risk function. It is reasonable to think though that an optimisation procedure restricted to small enough norms would still result in a well-defined ERM procedure, allowing for negative regularisation also in these case. 
    We leave a more thorough exploration of this effect for future work. We also mention that negative regularisation has been considered in a series of recent works \cite{wu, Kobak, hastie}.

    \section{Dependency of the generalisation  on the parameters} \label{sec:explore}

        \begin{figure}[t!]
        \centering
        \includegraphics[width=0.475\columnwidth]{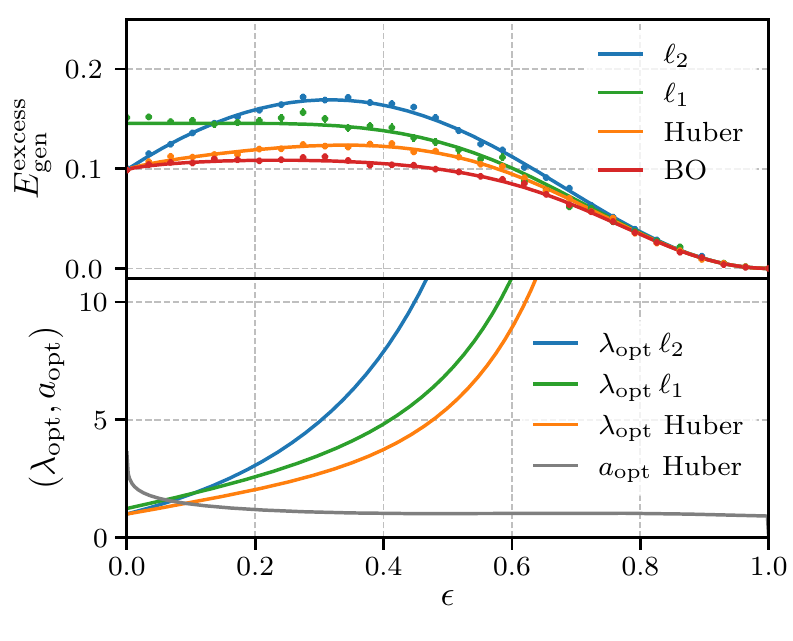}
        \hfill
        \includegraphics[width=0.475\columnwidth]{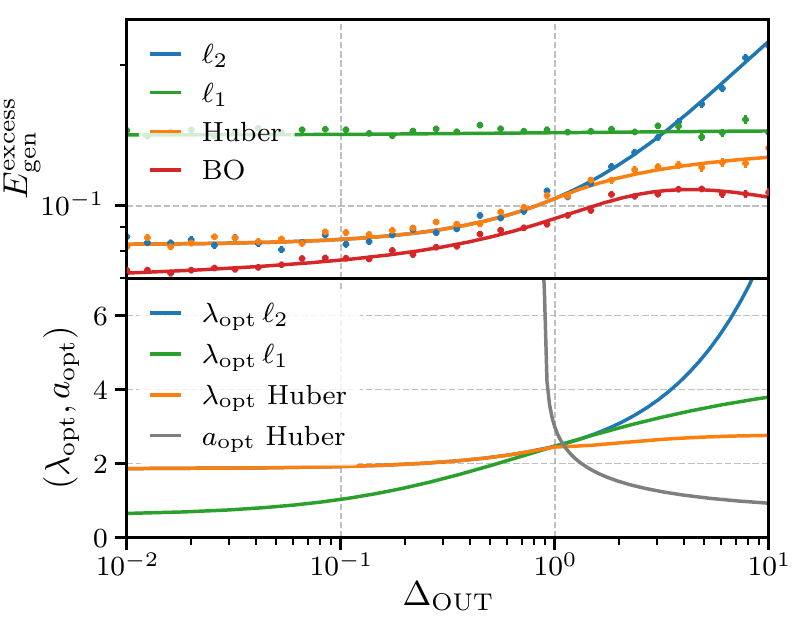}
        \hfill
        \caption{
        Excess generalisation error (top panels) as a function of the percentage of outliers' $\epsilon$ (left) and of the outliers' variance $\dout$ (right), along with the respective optimal hyper-parameters $(\lambda, a)$ (bottom panels). 
        In both plots $\alpha = 10$, $\beta = 0$, $\din = 1$, and for all numerical simulations (dots) have $d=200$, and are averaged over $1000$ samples.
        \textbf{(Left)} In this plot $\dout = 5$. 
        We observe that for small $\epsilon$ both the $\ell_2$ and the Huber loss converge to BO performance, while $\ell_1$ does not.
        \textbf{(Right)} In this plot $\epsilon = 0.3$. For $\dout < \din$ we notice that the Huber scale parameter diverges, so the Huber and $\ell_2$ loss become identical.
        Furthermore, the BO estimator has a non-monotonic behaviour as a function of $\dout$, while ERM estimators are monotonic.
        }
        \label{fig:figure-2}
                \vspace{-0.25cm}
    \end{figure}

    We now investigate the dependence of on the parameters characterising the amount of corruption caused by outliers. As in the previous section, we plot the excess generalisation error, and we always consider optimally tuned hyperparameters $\lambda$ and $a$. Again, we discuss the phenomenology arising for a particular value for the parameters of the model, but we stress that the phenomenology is qualitatively the same for other values of the parameters, see Appendix~\ref{SM:exploration} for additional experiments, as well as Appendix~\ref{SM:real} for a more realistic scenario.
    
    \paragraph{Dependency of the generalisation error on the amount of outliers.}
    We investigate the influence of the percentage of outliers $\epsilon$ on the generalisation error in Figure~\ref{fig:figure-2} left panel. We identify two qualitatively different regimes where the behaviour is different: the regime with small $\epsilon$ where outliers are rare and vice-versa the regime with $\epsilon$ close to 1, where the samples that would be considered outliers in the $\epsilon\to 0$ limit dominate the dataset. 
    
    We observe that in the small $\epsilon$ regime, for all losses, the ERM excess generalisation error increases as the percentage of outliers increases. In the limit of small $\epsilon$, the Huber loss performs similarly to $\ell_2$, which is known to be Bayes optimal at $\epsilon=0$ if optimally regularised \cite{cover1991information}. In contrast, the $\ell_1$ loss converges to a higher value for the generalisation error not reaching BO performances. For all finite values $\epsilon$, ERMs fails to reach BO performance for all three choices of the loss we studied. 
    
    By expanding the explicit asymptotic form for the generalisation error of ERM with $\ell_2$ loss presented in Corollary \ref{thm:l2sol}, we are able to provide an analytical characterisation of the small $\epsilon$  behaviour of $\ell_2$ ERM.
    \begin{theorem}\label{thm:excessL2epsilon}
        As $\epsilon\to0$ (taken after the $n, d \to \infty$ limit), we have that the generalisation error of optimally-regularised $\ell_2$~loss and the optimal regularisation satisfy
        \begin{equation}
        \Egen^{\ell_2} = \Egen^{\mathrm{BO}} + \kappa\epsilon + \mathcal{O}(\epsilon^2)\,,
        \quad
        \lambda = \din + \lambda_1 \epsilon + \mathcal{O}(\epsilon^2)
        \end{equation}
        where $\Egen^{\mathrm{BO}}$ is the Bayes optimal generalisation error, and $\kappa$ and $\lambda_1$ are functions of the model parameters that we provide in Appendix~\ref{SM:thm-L2-small-eps}.
    \end{theorem}

    In the $\epsilon \gtrsim 0.5$ regime, where outliers are dominant, we observe a decrease in the excess generalisation error, which eventually vanishes as $\epsilon$ approaches one. 
    In this limit the majority of the dataset is composed of data generated from a $\beta$-rescaled teacher, so that ERM will be biased to smaller ($0< \beta < 1$) or larger ($\beta > 1$) norms than if trained on a dataset with no outliers. As discussed previously, the optimal $\ell_2$ regularisation will counterbalance as much as possible this effect, aiming at the BO norm.
    In the extreme case where $\beta = 0$, outlier samples provide no meaningful information on the ground truth and the BO estimator equals the null vector $\boldsymbol{w} = 0$. This is compatible with what we observe in the figure, where $\lambdaopt$ diverges as $\epsilon$ approaches one.

    \paragraph{Dependency of the generalisation error on the variance of outliers.}
    In Figure~\ref{fig:figure-2}, right panel, we consider the dependency of the generalisation error as a function of the variance of the outliers $\dout$.
    We start by noticing that all ERM estimators have a generalisation error that grows monotonically with $\dout$, while the BO generalisation error has a non-monotonic behaviour, with a local maximum. This behaviour is also present for $\beta > 0$.
    Moreover, we notice that for large outliers' variance the $\ell_2$ loss performs worse than the more resilient $\ell_1$ and Huber losses, as expected.
    
    \paragraph{Performance difference between Huber and \texorpdfstring{$\ell_2$}{L2} loss}

    We observe an interesting behaviour in comparing the Huber~loss and the $\ell_2$~loss. Observing Figure~\ref{fig:figure-2}, right panel,
    we observe that as $\dout$ decreases the scale parameter value $a_{\rm opt}$ of the Huber~loss diverges to infinity at a finite value of $\dout$.  
    When $a \to +\infty$ the Huber~loss converges to the $\ell_2$~loss, as confirmed by the fact that their predicted generalisation errors coincides (Figure~\ref{fig:figure-2} right).
    It is remarkable that optimally-regularised Huber~loss is not always strictly better than the simpler $\ell_2$~loss in a whole region of model parameters.

    We explore further this phase transition by plotting the difference between the generalisation error of the optimal $\ell_2$ and Huber~losses as a function of both $\epsilon$ and $\dout$ in Figure~\ref{fig:figure-3}. We remark that the two plots in Figure~\ref{fig:figure-2} correspond respectively to horizontal and vertical slices in this plot.
    Also, we see that the difference in performance between the two losses is relevant just in the region of large $\dout$ and intermediate $\epsilon$. Moreover, focusing on the region of small $\epsilon$, large $\dout$ and large $\alpha$, \textit{i.e.} for really rare and noisy outliers, Figure~\ref{fig:figure-3} suggests that the more samples one has (larger $\alpha$), the less convenient it becomes to use the Huber loss, at fixed value of all other parameters.
    
    \begin{figure}[t!]
        % \centering
        \begin{minipage}[c]{0.43\columnwidth}
            \includegraphics[width=\textwidth]{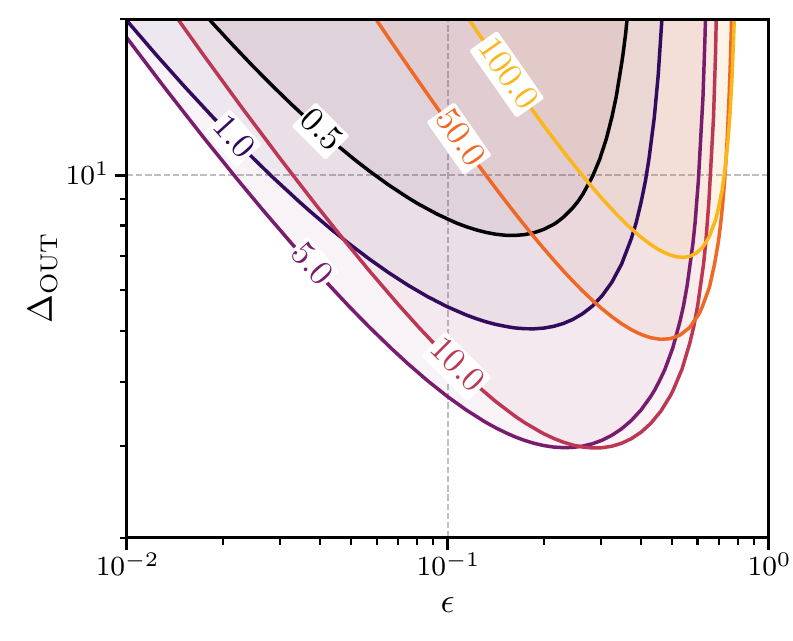}
        \end{minipage}
        \hfill
        \begin{minipage}[c]{0.4\columnwidth}
            \caption{\raggedright
            Difference between the generalisation error of optimally-regularised $\ell_2$ and Huber loss losses as a function of $\epsilon$ and $\dout$ for $\din = 1$, $\beta = 0$. Each shaded region denotes where the two losses have a non-zero difference, \textit{i.e.} where it is better to use the Huber loss, for different values of $\alpha$ as indicated by the labels.
            Focusing on the region of small $\epsilon$ and large $\dout$, \textit{i.e.} when dealing with just a few extremely noisy outliers, the plot suggests that having many samples (large $\alpha$) makes using the Huber loss not advantageous.
            }
        \label{fig:figure-3}
        \end{minipage}
        \hfill
        \vspace{-0.75cm}
    \end{figure}

\section{Concluding remarks}

We studied robust ERM in a model of data with outliers, focusing on three commonly used losses ($\ell_2$-regularised $\ell_2$, $\ell_1$, and Huber loss). The percentage of outliers, variance, and typical norm can be varied, allowing for a thorough exploration of the effect of outliers on the regression task. We derived analytically exact asymptotics for the generalisation and estimation errors of ERM estimators trained on this data model in the high-dimensional limit, where the number of samples $n$ and the dimension $d$ are large with a fixed finite ratio, and we compared them with BO estimators.

For the generalisation error, we found that optimally regularised ERM is asymptotically consistent in the large sample complexity limit at the condition of performing a norm calibration. For the estimation error, we showed that due to a norm calibration mismatch, the consistency of the estimator requires an oracle estimate of the optimal norm or the presence of a cross-validation set not corrupted by the outliers. We also observed an interesting phase transition as a function of the outlier parameters. In a full region of the parameter space $(\epsilon, \dout)$, the optimal scale parameter of the Huber loss $a_{\text{out}}$ diverges, and the error of the ERM with Huber and $\ell_2$ losses coincides.

Our work opens the way to the study of more complicated settings, for instance random features regression \cite{rahimi2007random} or deep learning in the lazy/neural tangent kernel regime \cite{chizat2019lazy}, where one could study effects that are not directly explorable in the vanilla generalised linear regression setting, such as over-parametrization \cite{belkin2020two}, benign overfitting \cite{bartlett2020benign}, implicit regularisation \cite{gunasekar2017implicit}, and their interplay with outliers.

\section*{Acknowledgements} 

We thank Lenka Zdeborová for fruitful discussions and insightful ideas, and Yatin Dandi and Borja Rodriguez Mateos for useful discussions in the early stages of this project. We acknowledge funding from the ERC under the European Union’s Horizon 2020 Research and Innovation Program Grant Agreement 714608-SMiLe, as well as by the Swiss National Science Foundation grants SNSF OperaGOST $200021\_200390$ and SNSF TMPFP2\_210012.

\newpage
\appendix

\section{Mathematical details}

\subsection{Claim: Bayes optimal estimator for the generalisation error}
\label{SM:bo-gen-error}

Define the generalisation error as in eq.~(\ref{eq:gen-error})
 \begin{equation}
        \Egen(\boldsymbol{w}) = \EEb{ (\bar{\boldsymbol{x}}, \bar{y} | \wstar) }{ \qty(\bar{y} - \frac{\bar{\boldsymbol{x}}\cdot \boldsymbol{w}}{\sqrt{d}} )^2} \,, 
    \end{equation}
where the average is over a sample-label pair $(\bar{\boldsymbol{x}}, \bar{y})$ generated with the same distribution as the training set, namely with the same teacher vector $\wstar$.
The posterior-averaged generalisation error is
\begin{equation}
    \mathcal{R}(\boldsymbol{w}) = \EEb{\wstar | \mathcal{D}}{
        \EEb{ (\bar{\boldsymbol{x}}, \bar{y}) | \wstar}{
            \qty(\bar{y} - \frac{\bar{\boldsymbol{x}}\cdot \boldsymbol{w}}{\sqrt{d}} )^2
        } 
    }
\end{equation}
where $\mathcal{D}$ is the observed training dataset, $\EEb{\wstar | \mathcal{D}}{\cdot}$ is the posterior average and $(\bar{\boldsymbol{x}}, \bar{y})$ is a sample label pair.
Differentiating and setting the derivative to zero (in order to minimise the posterior-averaged generalisation error) gives that the minimiser $\what$, \textit{i.e.} the BO estimator, satisfies 
\begin{equation}
    \what^{\rm BO}_{\rm gen}( \mathcal{D}) = \EEb{ \boldsymbol{w} | \mathcal{D} }{ \EEb{ \bar{\boldsymbol{x}}, \bar{y} | \boldsymbol{w} }{ \bar{y} \, \bar{\boldsymbol{x}} } } \,,
\end{equation}
which for our specific choice of data model reduces to
\begin{equation}
     \what^{\rm BO}_{\rm gen}( \mathcal{D}) = \qty(1-\epsilon + \beta \epsilon) \EEb{ \boldsymbol{w} | \mathcal{D}}{ \boldsymbol{w} } \, .
\end{equation}

\subsection{Theorem 4.1: self-consistent equations for the errors of ERM}
\label{SM:thm-ERM}

The more general proof of the form of the self-consistent equations can be found in \cite{Loureiro_2022}. Our case corresponds to taking $\boldsymbol{\varphi}_s(\boldsymbol{x}) = \boldsymbol{\varphi}_t(\boldsymbol{x}) = \boldsymbol{x}$ which is equivalent to the setting  presented in \cite{aubin2020generalisation}, of which we will follow the notation.

In the following we will refer to the loss function as $\mathcal{L}(\cdot,\cdot)$ and to the regularisation function as $r(\cdot)$. We then define the general quantities that will appear in the fixed point equations, they are:
\begin{equation}\label{eq.app-ingredients}
\begin{aligned}
    \Zoutstar(y,\omega,V) &= \mathbb{E}_{z\sim\mathcal{N}(\omega,V)}\left[\pout{y | z}\right] \, , \\
    \foutstar(y,\omega,V) &= \partial_\omega \log (\mathcal{Z}_{\mathrm{out}^*}(y,\omega,V)) \, ,\\
    \fout(y,\omega,V) &= (\mathcal{P}_V[\mathcal{L}(y,\cdot)](\omega) - \omega) / V \, ,\\
    \mathcal{Z}_{\mathrm{w}^*}(\gamma, \Lambda) & =\mathbb{E}_{\omega \sim P_{\mathbf{w}^*}}\left[e^{-\frac{1}{2} \Lambda \omega^2+\gamma \omega}\right] \,, \\
    f_{\mathrm{w}^*}(\gamma, \Lambda) &=\partial_\gamma \log \mathcal{Z}_{\mathrm{w}^*}(\gamma, \Lambda) \, ,\\
    f_{\mathrm{w}}(\gamma, \Lambda) &= \mathcal{P}_{\Lambda^{-1}}[r(.)]\left(\Lambda^{-1} \gamma\right),\\
\end{aligned}
\end{equation}
where $\omega$ in the definition of $\mathcal{Z}_{\mathrm{w}^*}(\gamma, \Lambda)$ is sampled with the teacher weight distribution and $\mathcal{P}_V[f](x)$ is the proximal operator defined as follows
\begin{equation}
    \mathcal{P}_V[f(y,\cdot)](x) = \arg\min_z \left[f(y,z) + \frac{1}{2V}(z-x)^2  \right]\,.
\end{equation}
The self-consistent equations in general form then read 
\begin{equation}\label{eq:fixed-point-equations-full-generality}
\begin{aligned}
    m &=\mathbb{E}_{\xi}\left[\mathcal{Z}_{\mathrm{w}^{\star}}(\sqrt{\hat{\eta}} \xi, \hat{\eta}) f_{\mathrm{w}^{\star}}(\sqrt{\hat{\eta}} \xi, \hat{\eta}) f_{\mathrm{w}}\left(\hat{q}^{1 / 2} \xi, \hat{\Sigma}\right)\right], \\
    q &=\mathbb{E}_{\xi}\left[\mathcal{Z}_{\mathrm{w}^{\star}}(\sqrt{\hat{\eta}} \xi, \hat{\eta}) f_{\mathrm{w}}\left(\hat{q}^{1 / 2} \xi, \hat{\Sigma}\right)^{2}\right], \\
    \Sigma &=\mathbb{E}_{\xi}\left[\mathcal{Z}_{\mathrm{w}^{\star}}(\sqrt{\hat{\eta}} \xi, \hat{\eta}) \partial_{\gamma} f_{\mathrm{w}}\left(\hat{q}^{1 / 2} \xi, \hat{\Sigma}\right)\right]\,, \\
    \hat{m} &=\alpha\int \mathbb{E}_{\xi}\left[\mathcal{Z}_{\mathrm{out}^{\star}}(y, \sqrt{\eta} \xi, (1-\eta)) \cdot f_{\mathrm{out}^{\star}}\left(y, \sqrt{\eta} \xi, (1-\eta)\right) f_{\mathrm{out}}\left(y, q^{1 / 2} \xi, \Sigma\right)\right]\,dy, \\
    \hat{q} &=\alpha \int\mathbb{E}_{\xi}\left[\mathcal{Z}_{\mathrm{out}^{\star}}\left(y, \sqrt{\eta} \xi, (1-\eta)\right) f_{\mathrm{out}}\left(y, q^{1 / 2} \xi, \Sigma\right)^{2}\right]dy\,, \\
    \hat{\Sigma}&=-\alpha \int\mathbb{E}_{ \xi}\left[\mathcal{Z}_{\mathrm{out}^{\star}}\left(y, \sqrt{\eta} \xi, (1-\eta)\right) \partial_{\omega} f_{\mathrm{out}}\left(y, q^{1 / 2} \xi, \Sigma\right)\right] dy\,.
\end{aligned}
\end{equation}
where we have denoted with $\mathbb{E}_\xi$ the expectation value over a standard Gaussian variable $\xi \sim  \mathcal{N}(0,1)$, and we defined $\eta = m^2 / q$ and $\hat\eta = \hat m^2/\hat q$.

In the main text we fixed the regularisation function to the $\ell_2$ regularisation, \textit{i.e.} $r(\mathbf{w}) = \frac{1}{2} \norm{\mathbf{w}}^2_2$ and the teacher weights to be extracted from the distribution $\wstar\sim\mathcal{N}(0, \mathbbm{1}_d)$. 
The quantities $\Zwstar$ and $\fwstar$ can be written as
\begin{equation}
    \mathcal{Z}_{\mathrm{w}^*}(\gamma, \Lambda) = \frac{e^{\frac{\gamma^2}{2(\Lambda+1)}}}{\sqrt{\Lambda+1}}\,,\quad
    f_{\mathrm{w}^*}(\gamma, \Lambda) = \frac{\gamma}{\Lambda + 1}\,.
\end{equation}

Also we have that 
\begin{equation}
    \Zw^{\ell_2,\lambda} (\gamma, \Lambda) = \frac{e^{\frac{\gamma^2}{2(\lambda + \Lambda)}}}{\sqrt{\lambda + \Lambda}}\,, \quad \fw^{\ell_2,\lambda} (\gamma,\Lambda) = \frac{\gamma}{\lambda+\Lambda}\,.
\end{equation}

This allows to write the first three self consistent equations as:
\begin{equation}
    m = \frac{\hat m}{\lambda + \hat\Sigma}\,, \quad q = \frac{\hat m^2 + \hat q}{(\lambda + \hat\Sigma)^2}\,, \quad \Sigma = \frac{1}{\lambda+\hat \Sigma}\,.
\end{equation}
As explained in \cite{aubin2020generalisation} these first three equations only depend on the regularisation term and they will be thus fixed during our analysis.

Thus, to compute the performance of ERM with a given loss, we just need to compute $\Zoutstar, f_{\mathrm{out}^*}$ and $f_{\mathrm{out}}$, and solve the corresponding system of self-consistent equations. For our noise model $P_{\rm out}$ we have that
\begin{equation}
    \Zoutstar(y,\omega,V) = \mathcal{N}_y(\omega, V + \din) + \mathcal{N}_y(\beta \omega, V + \beta^2 \dout) \,,
\end{equation}
and $\foutstar$ is given from the definition in eq.~(\ref{eq.app-ingredients}).

In the following subsections we specialise these equations to the $\ell_2$, $\ell_1$ and Huber~losses by computing $\fout$, which is the only loss-dependent part in the equations.

\paragraph{$\ell_2$ loss case.}
In this case, the $\fout$ in eq.~(\ref{eq.app-ingredients}) reads 
\begin{equation}\label{eq:app-fout-ell-2}
    f_{\mathrm{out}}^{\ell_2}(y, \omega, V)=\frac{y-\omega}{1+V}\,.
\end{equation}
Then, the remaining integrals can be solved explicitly as they are all expectation of Gaussian variables, giving
\begin{equation}\label{eq:l2-decorrelated-eq}
    \begin{aligned}
        m &= \frac{\hat m}{\lambda + \hat\Sigma}\,, \quad q = \frac{\hat m^2 + \hat q}{(\lambda + \hat\Sigma)^2}\,, \quad \Sigma = \frac{1}{\lambda+\hat \Sigma}\,, \\
        \hat m &= \frac{\alpha}{1 + \Sigma}(1 + \epsilon(\beta - 1)) \,, \quad \hat{\Sigma} = \frac{\alpha}{1 + \Sigma} \,, \\
        \hat q &= \frac{\alpha}{(1 + \Sigma)^2} \qty[ 1 + q + \deff + \epsilon \qty(\beta^2 - 1) - 2 m \qty(1+ \epsilon (\beta - 1))] \,.
    \end{aligned}
\end{equation}
For this specific case we have that one can find an explicit form for the order parameters which is presented in Appendix~\ref{SM:corollary-L2}.

\paragraph{Useful Integrals.}
To perform the double integration over $(y, \xi)$ for the last three equations in eq.~(\ref{eq:fixed-point-equations-full-generality}) we will use a linear change of variables. The relevant change of variables is $(u,v) = (y-\beta \sqrt{\eta} \xi, y - \sqrt{q} \xi)$. The associated Jacobian is constant and equal to $\abs{\beta \sqrt{\eta} - \sqrt{q}}$. 
With this change of variable, we obtain explicit expressions for the following integrals, that will be used below.
The first integral integral is
\begin{equation}
\begin{aligned}
    \mathcal{I}_{1}&(\Delta, \beta, \kappa, \Lambda, a, m, q) = \int_{\mathbb{R}^2} \dd{\xi} \dd{y} e^{-\frac{1}{2} \xi^2-\frac{1}{2} \frac{\qty(\beta \sqrt{\eta} \xi - y)^2}{\Delta + \beta^2 (1- \eta)}} \qty(\beta \sqrt{\eta} \xi - y)
    \begin{cases}
        a & \text{if } y- \sqrt{q} \xi > \kappa \\
        \frac{y- \sqrt{q} \xi}{\Lambda} & \text{if } \abs{y- \sqrt{q} \xi} < \kappa \\
        -a & \text{if } y- \sqrt{q} \xi < -\kappa \\
    \end{cases} \\
    &= \frac{2\pi}{\Lambda}\qty(\qty(1-\eta)\beta^2 + \Delta)^\frac{3}{2} \qty[ \frac{\sqrt{2}(\kappa- a\Lambda)e^{-\frac{\kappa ^2}{2 \qty(\beta ^2+\Delta -2 \beta  m+q)} }}{\sqrt{\pi (q-2m\beta+\beta^2+\Delta)}} -\erf\qty(\kappa / \sqrt{2(\beta ^2+\Delta -2 \beta  m+q)}) ]\,.
\end{aligned}
\end{equation}
The second integral is:
\begin{equation}
\begin{aligned}
    \mathcal{I}_{2}&(\Delta, \beta, \kappa, \Lambda, a, m, q) = \int_{\mathbb{R}^2} \dd{\xi} \dd{y} e^{-\frac{1}{2} \xi^2-\frac{1}{2} \frac{\qty(\beta \sqrt{\eta} \xi - y)^2}{\Delta + \beta^2 (1- \eta)}} \begin{cases}
        \qty(\frac{y- \sqrt{q} \xi}{\Lambda})^2 & \text{if } \abs{y- \sqrt{q} \xi} < \kappa \\
        a^2 & \text{otherwise} \\
    \end{cases} \\
    &=\frac{2\pi}{\Lambda^2}\sqrt{\qty(1-\eta)\beta^2 + \Delta} \Bigg[a^2 - \kappa \sqrt{2(q-2m\beta + \beta^2 + \Delta) / \pi} \, e^{-\frac{\kappa ^2}{2 \qty(\beta ^2+\Delta -2 \beta  m+q)}} \\
    &\qquad + \qty(q-2m\beta + \beta^2 + \Delta - a^2) \erf\qty( \kappa/\sqrt{2(\beta ^2+\Delta -2 \beta  m+q)} ) \Bigg] \, .
\end{aligned}
\end{equation}
The third  integral is 

\begin{equation}
\begin{aligned}
    \mathcal{I}_{3}&(\Delta, \beta, \kappa, \Lambda, a, m, q) = \int_{\RR^2}\dd{\xi}\dd{y} e^{-\frac{1}{2}\xi^2 -\frac{1}{2} \frac{\qty(y-\beta \sqrt{\eta} \xi)^2}{\Delta + \beta^2 (1- \eta)}} \mathbb{I}_{\abs{y-\sqrt{q}\xi} < \kappa} \\
    &=2\pi\sqrt{\qty(1-\eta)\beta^2 + \Delta} \erf\qty(\kappa / \sqrt{2(\beta ^2+\Delta -2 \beta  m+q)}) / \Lambda \,,
\end{aligned}
\end{equation}

Using these integrals
we can write explicitly the self-consistent equations, leading to the equations presented in Theorem \ref{thm:main-results}.
\paragraph{$\ell_1$ loss case.}
For the $\ell_1$ loss, the $\fout$ in eq.~(\ref{eq.app-ingredients}) reads 
\begin{equation}
    \fout^{\ell_1}(y, \omega, V) = 
    \begin{cases}
        \frac{1}{V}\qty(y-\omega) & \text{if } |\omega - y| \leq V \\ 
        \operatorname{sign}(\omega-y) & \text{otherwise }
    \end{cases}\,.
\end{equation}
With the integration technique explained above we have that the explicit form of the self-consistent equations is
\begin{equation}
    \begin{aligned}
        \hat m &= \frac{\alpha}{\Sigma} \qty[(1-\epsilon) \erf\qty(\chiin) + \beta \epsilon\erf\qty(\chiout)] 
        \, , \\
        \hat q &= \frac{\alpha}{\Sigma ^2} \Bigg[(1-\epsilon) \left(\zetain - \mu ^2\right) \erf\qty(\chiin) + \epsilon \left(\zetaout-\mu ^2\right) \erf\qty(\chiout)
        \\
        &\quad 
        - \mu  \sqrt{\frac{2}{\pi }} \left( (1-\epsilon ) \sqrt{\zetain} \, e^{-\chiin^2}\right. + \left.\epsilon \sqrt{\zetaout} \, e^{-\chiout^2} \right) + \mu^2\Bigg] \, , \\
        \hat \Sigma &= \frac{\alpha}{\Sigma }  \qty[(1-\epsilon ) \erf\qty(\chiin) + \epsilon \erf\qty(\chiout)] \,,
    \end{aligned}
\end{equation}
where, as in the main text, we defined $\zetain = \din -2 m+q+1$, $\zetaout = \dout+\beta^2+q - 2\beta m$ and
\begin{equation}
    \mu^{\ell_1} = \nu^{\ell_1} = \Sigma
    \, , \quad
    \chiin^{\ell_1} = \Sigma / \sqrt{2 \zetain}
    \, , \quad
    \chiout^{\ell_1} = \Sigma / \sqrt{2 \zetaout}
    \, ,
\end{equation}

\paragraph{Huber loss case.} For the Huber loss, the $\fout$ in eq.~(\ref{eq.app-ingredients}) reads 
\begin{equation}
    \fout^{\mathrm{Huber}}(y, \omega, V) = 
    \begin{cases}
        -a & \text{if } y-\omega>a(V+1) \\ 
        \frac{y-\omega}{V+1} & \text{if }|y-\omega|<a(V+1) \\ 
        a & \text{otherwise}
    \end{cases} \,.
\end{equation}

we have that the integrals in eq.~(\ref{eq:fixed-point-equations-full-generality}) become
\begin{equation}
    \begin{aligned}
        \hat m &= \frac{\alpha}{\Sigma + 1} \qty[(1-\epsilon) \erf\qty(\chiin) + \beta \epsilon\erf\qty(\chiout)] 
        \, , \\
        \hat q &= \frac{\alpha}{\qty(\Sigma + 1) ^2} \Bigg[(1-\epsilon) \left(\zetain - \mu ^2\right) \erf\qty(\chiin) + \epsilon \left(\zetaout-\mu ^2\right) \erf\qty(\chiout)
        \\
        &\quad 
        - \mu  \sqrt{\frac{2}{\pi }} \left( (1-\epsilon ) \sqrt{\zetain} \, e^{-\chiin^2}\right. + \left.\epsilon \sqrt{\zetaout} \, e^{-\chiout^2} \right) + \mu^2\Bigg] \, , \\
        \hat \Sigma &= \frac{\alpha}{\Sigma + 1 }  \qty[(1-\epsilon ) \erf\qty(\chiin) + \epsilon \erf\qty(\chiout)] \,
    \end{aligned}
\end{equation}
where $\zetain = \din -2 m+q+1$, $\zetaout = \dout+\beta^2+q - 2\beta m$
and
\begin{equation}
    \nu^{\text{Huber}} = \Sigma + 1
    \, , \quad
    \mu^{\text{Huber}} = a\nu^{\text{Huber}} 
    \, , \quad
    \chiin^{\text{Huber}} = \mu^{\text{Huber}} / \sqrt{2\zetain} 
    \, , \quad
    \chiout^{\text{Huber}} = \mu^{\text{Huber}} / \sqrt{2\zetaout}
    \, .
\end{equation}

\subsubsection[Corollary: explicit form of the error for L2]{Corollary: explicit form of the error for \texorpdfstring{$\ell_2$}{Lg}}
\label{SM:corollary-L2}

\begin{corollary}\label{thm:l2sol}
    In the Ridge~regression case, for a generic value of $\beta$, with fixed regularisation $\lambda \geq 0$ the generalisation error has the explicit form
    \begin{equation}\label{eq:explicit-solution-l2-main}
        \begin{aligned}
        \Egen =&\ 1 + (\beta^2 - 1)\epsilon + q - 2m (1 + (\beta - 1) \epsilon) \,, \\ 
        m =&\ \frac{2 \alpha  \Gamma }{p +t+1} \,,\\
        q =&\ \frac{4 \alpha  \qty[\alpha  \Gamma ^2 (p +t-3)+\Lambda  (p +t+1)]}{(p+t+1) \qty[p^2-2 c +t^2+2 t (p+1)+1]} \,,
        \end{aligned}
    \end{equation}
    where $\deff = (1-\epsilon)\din + \epsilon \dout$, $p=\alpha+\lambda$, $c=\alpha-\lambda$, $t =\sqrt{ (p-1)^2 + 4\lambda}$, $\Gamma = 1+\epsilon(\beta - 1)$ and $\Lambda = 1 + \deff +\epsilon (\beta^2 - 1)$.
\end{corollary}

\begin{proof}
The form for the excess generalisation error is the same as in Theorem~\ref{thm:main-results}, so we look for a solution for the self-consistent equations, eqs.~(\ref{eq:l2-decorrelated-eq}).
The equations can be simplified if one defines $\Gamma = 1 + \epsilon(\beta - 1)$ and $\Lambda = 1 + \deff + \epsilon \qty(\beta^2 - 1)$ with $\deff = (1-\epsilon) \din + \epsilon \dout$. To solve them we first notice that the equations for $\Sigma$ and $\hat{\Sigma}$ only depend on themselves and thus can be solved, getting:
\begin{equation}\label{eq:sigma-sigma-hat-ridge}
    \Sigma = \frac{1-\alpha-\lambda+t}{2 \lambda}\,,\quad \hat{\Sigma} = \frac{-1+\alpha-\lambda+t}{2} \, ,
\end{equation}
where here and in the following we have that $t=\sqrt{4 \lambda+(-1+\alpha+\lambda)^2}$. 
Plugging these values in the equations for $m$ and $\hat{m}$ gives
\begin{equation}
    m = \frac{2 \alpha  \Gamma }{\alpha +\lambda +t+1} \,, \quad \hat{m} = -\frac{2 \alpha  \Gamma  \lambda }{-\alpha +\lambda +t+1}
\end{equation}
Finally, plugging everything in the equations for $q$ and $\hat{q}$, we obtain
\begin{equation}
    \begin{aligned}
        q =& \frac{4 \alpha  \left(\alpha  \Gamma ^2 (p+t-3)+\Lambda 
   (p+t+1)\right)}{(p+t+1) \left(-2 c+(p+t)^2+2 t+1\right)} \\
        \hat{q} =& \frac{4 \alpha  \lambda ^2}{(2 \lambda -p+t+1)^2} \left(\frac{4 \alpha  \left(\alpha  \Gamma ^2 (p+t-3)+\Lambda  (p+t+1)\right)}{(p+t+1) \left(-2 c+(p+t)^2+2 t+1\right)}+\Lambda -\frac{4 \alpha  \Gamma ^2}{p+t+1}\right)
   \end{aligned}
\end{equation}
where we defined $p = \alpha + \lambda$ and $c = \alpha - \lambda$. Upon rearranging of terms we have the same form for $q$ as in the Corollary.
\end{proof}

\subsection{The rate of BO is \texorpdfstring{$1/\alpha$}{-1}}
\label{SM:rate-BO}

The equations for the BO are presented in the main text eq.~(\ref{eq:bayes-opt-fpes}). If we suppose that $q_{\mathrm{b}} = 1 - c / \alpha + \mathcal{O}(\alpha^{-2})$ and $\hat{q}_{\mathrm{b}} = \hat{c} \alpha + \mathcal{O}(1)$ we have that
\begin{equation}
    q_{\mathrm{b}} = \frac{\hat{q}_{\mathrm{b}}}{1 + \hat{q}_{\mathrm{b}}} = 1 - \frac{1}{\hat{c} \alpha}
\end{equation}
thus the relation between the expansion coefficients which is $\hat{c} = 1 / c$. The coefficient $\hat{c}$ is found from the first order expansion of the self-consistent equations and it is equal to
\begin{equation}
    \hat{c} = \int \dd{\xi} \dd{y} e^{-\frac{1}{2} \xi^2}
    \frac{
        \left(\frac{\epsilon  (\beta  \xi -y) e^{-\frac{(y-\beta  \xi )^2}{2 \dout }}}{\sqrt{2 \pi} \dout }+\frac{(\epsilon -1) (y-\xi ) e^{-\frac{(y-\xi )^2}{2 \din }}}{\sqrt{2 \pi } \din}\right)^2
    }{
        \frac{\epsilon  e^{-\frac{(y-\beta  \xi )^2}{2 \dout }}}{\sqrt{2 \pi } \sqrt{\dout }}-\frac{(\epsilon -1) e^{-\frac{(y-\xi )^2}{2 \din }}}{\sqrt{2 \pi }\sqrt{\din }}
    }
\end{equation}

Note that even if we supposed a faster decay for the second term of $q_{\mathrm{b}}$ we would find anyway that $\hat{q}_{\mathrm{b}} \propto \alpha$.

\subsection{Theorem 5.1: large sample size limit for the self-consistent equations for ERM}
\label{SM:thm-large-alpha}

In this Appendix we derive the asymptotic value of the generalisation and estimation errors of ERM for $\alpha \to \infty$. We will first consider the case for the $\ell_2$ loss where the explicit solution of the order parameters gives a more direct approach to study the above mentioned limit. After we will consider the $\ell_1$ and Huber~loss with a framework together.

\paragraph{$\ell_2$ loss.}
In this case we can leverage Corollary~\ref{thm:l2sol} to get the large $\alpha$ limit of the order parameters and thus of the generalisation error. We notice that the solution presented in Corollary~\ref{thm:l2sol} are formal solution that do not suppose anything on the scalings in $\alpha$ of the other parameters, \textit{i.e.} $\lambda$. We will consider the case where $\lambda$ is constant in $\alpha$ and see that this suffices to have BO rates for the generalisation error. From eq.~(\ref{eq:explicit-solution-l2-main}) we have by expanding
\begin{equation}
    q = \Gamma ^2+\frac{\Lambda -\Gamma ^2 (2 \lambda +1)}{\alpha }+\frac{\Gamma ^2 \left(3 \lambda ^2-1\right)-2 \lambda  \Lambda
   +\Lambda }{\alpha ^2}+\mathcal{O}\qty(\frac{1}{\alpha^3}) \,,
\end{equation}
and
\begin{equation}
    m = \Gamma -\frac{\Gamma  \lambda }{\alpha }+\frac{\Gamma  (\lambda -1) \lambda }{\alpha ^2}+ \mathcal{O}\qty(\frac{1}{\alpha^3})\,,
\end{equation}
with the same definition of $\Gamma = 1 + \epsilon (\beta - 1)$ and $\Lambda = 1 + \deff + \epsilon (\beta^2 - 1)$. By the knowledge of the order parameters and the formula for the generalisation and estimation error one has that
\begin{equation}
    \begin{aligned}
        &\Egen = \frac{\deff -(\beta -1)^2 (\epsilon -1) \epsilon }{\alpha } \\
        &+\frac{2 (\beta -1)^2 \lambda  (\epsilon -1) \epsilon +((\beta -1) \lambda  \epsilon +\lambda )^2-\left((\beta -1)^2 (\epsilon -1) \epsilon \right)-2 \deff  \lambda +\deff }{\alpha ^2} + \mathcal{O}(\alpha^{-3}) \,,\\
        &\Eest = (\beta -1)^2 \epsilon ^2 + 
        \frac{
            (\beta -1) \epsilon  ((2 \lambda +1) (\epsilon -1) -\beta  (2 \lambda  \epsilon +\epsilon-1))+\deff
        }{\alpha} \\
        & + \frac{\lambda  ((\beta -1) \epsilon  (\lambda  (3 (\beta -1) \epsilon +4)-2 \beta )+\lambda)-\left((\beta -1)^2 (\epsilon -1) \epsilon \right)-2 \deff  \lambda +\deff }{\alpha ^2} + \mathcal{O}(\alpha^{-3})\,.
    \end{aligned}
\end{equation}
Notice that in the estimation error the value of the plateu depends only on the noise parameters $\epsilon$ and $\beta$ but not on the noise variances.

\paragraph{$\ell_1$ and Huber loss}
We make the following ansatz for the leading order form of the following parameters in $\alpha$
\begin{equation}\label{eq.ansatz}
    m = m_0\,, \qquad q=q_0\, , \qquad \Sigma= \frac{\Sigma_0}{\alpha} \, , \qquad \hat{m}=\hat{m}_0\alpha \, , \qquad \hat{q}=\hat{q}_0 \alpha  \, , \qquad \hat{\Sigma} = \hat{\Sigma}_0 \alpha \, , 
\end{equation}
where $m_0, q_0, \Sigma_0, \hat{m}_0, \hat{q}_0$ and $\hat{\Sigma}_0$ are constants not depending on $\alpha$ in the limit. We use the following ansatz for $\lambda$ as a function of $\alpha$
\begin{equation}
    \lambda = \lambda_0 + \lambda_1 \alpha\,, 
\end{equation}
again with $\lambda_1$ and $\lambda_2$ finite in the large $\alpha$ limit.
We now show that this ansatz leads to a system of self-consistent equations involving only  $m_0, q_0, \Sigma_0, \hat{m}_0, \hat{q}_0$ and  $\lambda_1,  \lambda_0$ and additionaly find the conditions for which the first order term in $\lambda$ vanishes.
This system characterises at leading order the behaviour of the errors of ERM.

From the loss-independent self-consistent equations presented in Theorem \ref{thm:main-results}
\begin{equation}
    m = \frac{\hat m}{\lambda + \hat\Sigma}\,, \quad q = \frac{\hat m^2 + \hat q}{(\lambda + \hat\Sigma)^2}\,, \quad \Sigma = \frac{1}{\lambda+\hat \Sigma}\,
\end{equation}
we obtain at large $\alpha$ the equations
\begin{equation}\label{eq.q=m2}
    m_0 = \frac{\hat m_0}{\lambda_1 + \hat\Sigma_0}\,, \quad q_0 = \frac{\hat m_0^2}{(\lambda_1 + \hat\Sigma_0)^2}=m_0^2\,, \quad \Sigma_0 = \frac{1}{\lambda_1+\hat \Sigma_0}\,
\end{equation}
Notice that the expressions are independent of $\lambda_0$, which means that for $\lambda$ to have an effect on the generalisation error it needs to have a non-vanishing coefficient for the term linear in $\alpha$, \textit{i.e.} $\lambda_1 \neq 0$.
At leading order, the loss-independent constants defined in Theorem~\ref{thm:main-results} read 
\begin{equation}\label{eq:zeta-in-out-alpha-big}
    \zetainzero = \din + (m_0-1)^2\,,\qquad\zetaoutzero = \dout + (m_0-\beta)^2\, .
\end{equation}

\paragraph{$\ell_1$ loss.}
For the $\ell_1$ loss we have at leading order
\begin{equation}
    \chiin=\chiout=\mu=\nu=\mathcal{O}(1/\alpha)\,,
\end{equation}
giving for the three loss-dependent self-consistent equations
\begin{equation} \label{eq:firstorder_l1}
    \begin{aligned}
        \hat{m}_0 &= \sqrt{\frac{2}{\pi }} \left(\frac{1-\epsilon }{\sqrt{\zetainzero}}+\frac{\beta\epsilon }{\sqrt{\zetaoutzero}}\right) \, ,\\
        \hat{\Sigma}_0 &= \sqrt{\frac{2}{\pi }} \left(\frac{1-\epsilon }{\sqrt{\zetainzero}}+\frac{\epsilon }{\sqrt{\zetaoutzero}}\right) \, ,\\
        \hat{q}_0 &= 1\, .
    \end{aligned}
\end{equation}

\paragraph{Huber loss.}
For the Huber loss we will consider the scale parameter $a$ to be $\alpha$ independent. Thus we have at leading order
\begin{equation} \label{eq:chi_zero}
\chiin=a/\sqrt{2\zetainzero}\,, \quad
\chiout=a/\sqrt{2\zetaoutzero}\,, \quad
\mu=a\,, \quad
\nu=1\, ,
\end{equation}
giving for the three loss-dependent self-consistent equations
\begin{equation} \label{eq:firstorder_huber}
    \begin{aligned}
        \hat m_0 &= (1-\epsilon) \erf\qty(\frac{a}{\sqrt{2\zetainzero}}) +\beta \epsilon \erf\qty(\frac{a}{\sqrt{2\zetaoutzero}})\,, \\ 
        \hat \Sigma_0 &= (1-\epsilon ) \erf\qty(\frac{a}{\sqrt{2\zetainzero}}) + \epsilon \erf\qty(\frac{a}{\sqrt{2\zetaoutzero}}) \, ,
        \\
        \hat{q}_0 &=\ (1-\epsilon) \left(\zetainzero -a^2\right) \erf\left(\frac{a}{\sqrt{2 \zetainzero}}\right) + \epsilon  \left(\zetaoutzero -a^2\right) \erf\left(\frac{a}{\sqrt{2 \zetaoutzero}}\right)\,.
    \end{aligned}
\end{equation}
At leading order, the generalisation and estimation error read
\begin{equation}
\begin{aligned}
    \Egen \,
    &=   1 + \epsilon (\beta^2 - 1) + m_0^2 - 2  m_0  (1 + \epsilon (\beta - 1)) + \deff =\\
    &= (m_0-1)^2 - 2  m_0  \epsilon (\beta - 1) + \epsilon (\beta^2 - 1) + \deff\,, \\ 
    \Eest \,&= 1 + q_0 - 2m_0 = (m_0-1)^2 \,.
\end{aligned}
\end{equation}

\subsubsection{Rates of ERM}
\label{SM:rates-ERM}

Before deriving any conclusion from the set of self-consistent equations that we just derived in the large $\alpha$ limit, it is worth considering what is the scaling of the first subleading term for $m$ and $q$, as it determines the scaling of the first subleading term for the generalisation and estimation errors. 

It is easy to see that a general expansion in integer powers of alpha for the order parameters (so more general of the leading order ansatz of eq. (\ref{eq.ansatz})) still satisfies the self-consistent equations, as the order parameters enter the equations only through $\mathcal{C}^{\infty}$ functions. We just need to check that our functions are indeed smooth. This is contingent on the fact that $\Sigma_0 \neq 0$ and $\hat{\Sigma}_0 + \lambda_1\neq 0$. The first condition can just assumed to be true because of eq. \eqref{eq.q=m2}. Indeed, for $\Sigma_0$ to be zero we need $\hat{\Sigma}_0$ to diverge, which can't be since $\zetainzero>0$, $\zetaoutzero>0$ as long as $\din>0$, $\dout>0$ (equation \eqref{eq:zeta-in-out-alpha-big}). For the second one we can simply notice that as long as $\epsilon\leq1$, $\hat{\Sigma}_0 > 0$.

Thus, the subleading order of $m$ and $q$, and thus of the errors, is at least of order $\mathcal{O}(\alpha^{-1})$.

On the other hand, at the values of the parameters for which we have that the error of ERM converges to the BO performance, we have that the subleading order of the errors, is at most of order $\mathcal{O}(\alpha^{-1})$, which is the BO scaling.

Thus, whenever ERM is consistent, it also achieves BO rates. 
A more detailed computation would be needed to asses whether ERM is identical to BO at order $\mathcal{O}(\alpha^{-1})$, but we believe that this will not be the case in general.

\subsubsection{Angle between ERM minimiser and teacher goes to zero}
\label{SM:thm-large-alpha-angle}

The angle between the student and the teacher weights is given in concentrates in high-dimension onto \cite{aubin2020generalisation}
\begin{equation}\label{eq:angle-teacher-student-ovl}
    \theta_{\boldsymbol{w}^\star,\hat{\boldsymbol{w}}} = \frac{1}{\pi}\arccos\left( \frac{m}{\sqrt{q}} \right)
\end{equation}
but from the loss-independent equations eq. (\ref{eq.q=m2}) we have $q \to m^2$, so that $\theta \to 0$ for large $\alpha$.

\subsubsection{Consistency of generalisation error}
\label{SM:thm-large-alpha-gen}

In order for the generalisation error to be consistent, we need to impose that the excess generalisation error vanishes, \textit{i.e.} we require that
\begin{equation}
    \Egen = (m_0-1)^2 - 2  m_0  \epsilon (\beta - 1) + \epsilon (\beta^2 - 1) + \deff \overset{!}{=} \epsilon (1-\epsilon) (1-\beta)^2 + \deff
\end{equation}
Which results in $m_0 = \Gamma = 1 +\epsilon (\beta - 1)$. 
This imposes a condition on $\lambda_1$ because of eq.~(\ref{eq.q=m2}) which reads
\begin{equation}\label{eq:lambda_opt}
    \lambda_1 = \frac{\hat m_0}{\Gamma} - \hat{\Sigma}_0
\end{equation}

This needs to be compatible with $\lambda\geq 0$, which at large $\alpha$ implies $\lambda_1\geq 0$. Firstly we will consider the $\ell_2$~loss. In this case we have $\hat{m}_0 = \Gamma$ and $\hat\Sigma_0 = 1$, so $\lambda_1=0$ meaning that the value of $\lambda$ doesn't need to diverge with $\alpha$ to have consistency of $\Egen$. This result is compatible to the one obtained by looking at the explicit solution of the order parameters as in Corollary~\ref{thm:l2sol}, where we have found that even with a finite value of $\lambda$ in the limit $\alpha\to\infty$ the generalisation error has optimal rate.

For the other two losses the inequality in eq.~(\ref{eq:lambda_opt}) is equivalent to
\begin{equation}
    (1-\beta)(1-\epsilon)\epsilon\left[ \erf\qty(\frac{a}{\sqrt{2\zetainzero}}) - \erf\qty(\frac{a}{\sqrt{2\zetaoutzero}}) \right] \geq 0
\end{equation}
For $\beta < 1$, considering that $\epsilon \in (0,1)$, remembering the definitions in eq.~(\ref{eq:zeta-in-out-alpha-big}), this condition is equivalent to
\begin{equation}\label{eq:consitency-condition}
    \dout - \din \geq (1-\beta)^2(2\epsilon-1)\,.
\end{equation}
If $\beta > 1$ the condition in eq.~(\ref{eq:consitency-condition}) flips the order of the condition.

We remark that in the limit $a \to \infty$ the equations is always satisfied. This tells us that whenever L2 is better then Huber, the optimisation over $a$ of Huber will lead to $a \to \infty$, \textit{i.e.} the Huber loss becomes the $\ell_2$ loss.

\subsubsection{Consistency of estimation error}
\label{SM:thm-large-alpha-estim}

In order for the estimation error to be consistent, we need to impose that the estimation error vanishes for $\alpha\to\infty$. This is the condition
\begin{equation}
    \Eest = (m_0-1)^2 \overset{!}{=} 0\,,
\end{equation}
which is equivalent to $m_0 = 1$. As in Appendix~\ref{SM:thm-large-alpha-gen} we use eq.~(\ref{eq.q=m2}) and obtain the value of $\lambda_1$ as 
\begin{equation} \label{eq:estim_optimal}
    \lambda_1 = (\beta-1)\epsilon \erf\qty(\frac{a}{\sqrt{2\zetaoutzero}})\,,
\end{equation}
for the Huber loss.
The other two losses can be obtained by setting $a \to \infty$ for the $\ell_2$, and $a \to 0$ with the appropriate rescaling of $\lambda_1$ for $\ell_1$. Both limits do not alter the overall sign of $\lambda_1$, which just depends on $\beta - 1$.
This implies that the estimator is can only be consistent in estimation for $\beta > 1$ since we have the constraint $\lambda \geq 0$.

\subsubsection[Consistency of estimation error for L2 with negative regularisation]{Consistency of estimation error for \texorpdfstring{$\ell_2$}{Lg} with negative regularisation}
\label{SM:thm-large-alpha-L2-negative-reg}

In Appendix~\ref{SM:claim-L2-negative-lambda} we show that the Ridge~regression risk can endure negative regularisation without losing strict convexity for all $\lambda > - (1 - \sqrt{\alpha})^2$. 

Using eq.~(\ref{eq:estim_optimal}), in the limit $a \to \infty$ to recover the $\ell_2$ loss, we see that the condition above for lambda implies, at leading order in $\alpha$ and for the $\ell_2$ loss that
\begin{equation}
    (\beta-1)\epsilon > -1 \, .
\end{equation}

\subsubsection{Oracle norm/label rescaling restores consistency}
\label{SM:thm-large-alpha-rescaling}
We notice that since the angle between the teacher and the student goes to zero as shown in Appendix~\ref{SM:thm-large-alpha-angle}, it suffices to rescale the student optimally to achieve consistency.
This can be achieved either by knowing the norm, or by cross-validation on the appropriate validation set to find the best norm.

Notice that all losses satisfy some kind of homogeneity property, \textit{i.e.}
\begin{equation}
    \losshuber{a} \qty(c\,y, \frac{ \boldsymbol{x} \cdot \boldsymbol{w} }{\sqrt{d}}) = c^2 \, \losshuber{c\,a} \qty(y, \frac{ \boldsymbol{x} \cdot \boldsymbol{w} }{c \sqrt{d}}) \, ,
\end{equation}
\begin{equation}
    \ell_1 \qty(c\,y, \frac{ \boldsymbol{x} \cdot \boldsymbol{w} }{\sqrt{d}}) = c \, \ell_1 \qty(y, \frac{ \boldsymbol{x} \cdot \boldsymbol{w} }{c \sqrt{d}}) \, ,
\end{equation}
\begin{equation}
    \ell_2 \qty(c\,y, \frac{ \boldsymbol{x} \cdot \boldsymbol{w} }{\sqrt{d}}) = c^2 \, \ell_2 \qty(y, \frac{ \boldsymbol{x} \cdot \boldsymbol{w} }{c \sqrt{d}}) \, .
\end{equation}
For this reason, the rescaling of the ERM estimator can be produced equivalently by rescaling the labels of the training set, and the hyperparamters, appropriately.

\subsection{Claim: Ridge~regression risk is still convex for not-too-large negative regularisation}
\label{SM:claim-L2-negative-lambda}

Consider the Ridge~regression risk
\begin{equation}\label{eq:risk-def-ridge}
    \mathcal{R}_{\mathrm{ridge}}(\boldsymbol{w}) = \frac{1}{2} \sum_{i=1}^n \norm{ y_i - \frac{\boldsymbol{x}_i \cdot \boldsymbol{w}}{\sqrt{d}} }^2_2 + \frac{\lambda}{2} \norm{\boldsymbol{w}}^2_2 \, .
\end{equation}
This is a quadratic form in the $d$-dimensional vector $\boldsymbol{w}$, with quadratic part
\begin{equation}\label{eq:risk-ridge-quadratic-part}
    \mathcal{R}_{\mathrm{ridge}}^{\mathrm{quad.}}(\boldsymbol{w})
    = 
    \frac{1}{2} \sum_{a, b=1}^d {w}^a {w}^b \left( \frac{\alpha}{n} \sum_{i=1}^n {x}_i^a {x}_i^b + \lambda \delta_{ab} \right) 
    \overset{\mathrm{def}}{=}
    \frac{1}{2} \sum_{a, b=1}^d {w}^a {w}^b H_{a,b}
    \, ,
\end{equation}
where the Hessian of the quadratic form is the matrix $H = X + \lambda \, \mathrm{Id}$.
The matrix $X$ is a $d \times d$ Marchenko-Pastur random matrix.
Supposing $\alpha  = n/d > 1$, this matrix has minimal eigenvalue concentrating on $(1 - \sqrt{1 / \alpha})^2$, giving minimal Hessian eigenvalue equal to
\begin{equation}
    \text{minimal eigenvalue} = (1 - \sqrt{\alpha})^2 + \lambda \, .
\end{equation}
Thus, the Ridge~regression risk is convex as long as the Hessian is positive definite, meaning
\begin{equation}
    \forall \lambda \, : \,\lambda > - (1 - \sqrt{\alpha})^2 \, .
\end{equation}

\subsection[Theorem 6.1: small outliers' percentage epsilon limit of the error of ERM with L2 loss]{Theorem 6.1: small outliers' percentage \texorpdfstring{$\epsilon$}{Lg} limit of the error of ERM with \texorpdfstring{$\ell_2$}{Lg} loss}
\label{SM:thm-L2-small-eps}

We start from the knowledge of the explicit form of the order parameters in the Ridge~regression case, presented in Corollary~\ref{thm:l2sol}. If one expands $m$ and $q$  up to second order in $\epsilon$ one has that:
\begin{equation}
    m = m_0 + m_1 \epsilon + m_2 \epsilon^2 + \mathcal{O}(\epsilon^3)\,, \: q = q_0 + q_1 \epsilon + q_2 \epsilon^2 + \mathcal{O}(\epsilon^3)\,, \: \lambda = \lambda_0 + \lambda_1 \epsilon + \lambda_2 \epsilon^2 + \mathcal{O}(\epsilon^3)\,,
\end{equation}
with coefficients that are constants in $\epsilon$.
The generalisation error up to second order in $\epsilon$ is
\begin{equation}
\begin{aligned}
    \Egen &= {\Egen}_{0} + {\Egen}_{1} \epsilon + {\Egen}_{2} \epsilon^2 + \mathcal{O}(\epsilon^3) \\
    &= 1 + q + \left(\beta ^2-1\right) \epsilon + \deff - 2 m ((\beta -1) \epsilon +1) \\
    &= (1-2 m_0 + q_0 + \din) + ((\beta -1) (\beta - 2 m_0 + 1) + \dout - \din - 2 m_1 + q_1) \epsilon \\
    & + (1 - 2 (\beta -1) m_1 - 2 m_2) \epsilon^2  + \mathcal{O}\left(\epsilon ^3\right) \\
\end{aligned}
\end{equation}
This shows that the generalisation error has the correct expansion in powers of $\epsilon$. We now need to find the optimal regularisation $\lambda$, and plug it back in in the generalisation error to complete the proof.

To start, we notice that at $\epsilon = 0$ the BO generalisation error can be found explicitly, as the output channel become a single Gaussian, giving
\begin{equation}\label{eq:order-zero-gen-error}
\begin{aligned}
    {\Egen^{\mathrm{BO}}}_0 &= \frac{1}{2} \left(\sqrt{(\alpha +\din +1)^2-4 \alpha }-\alpha + \din + 1\right) \\ 
    {\Egen^{\ell_2}}_0 &=\frac{1}{2} \left(\frac{\lambda_0 (\alpha + \din) + \alpha (\alpha + \din - 2) + \din + \lambda_0 + 1}{\sqrt{(\alpha
   + \lambda_0 - 1)^2 + 4 \lambda_0 }} - \alpha + \din + 1\right) \\ 
\end{aligned}
\end{equation}
where in the second line we compare with the zero-th order term of $\ell_2$ ERM.
The two generalisation errors match for $\lambda_0 = \din$, giving us the zero-th order value of the expansion of the optimal regularisation.

To find the next orders for the optimal regularisation $\lambda$ we compute order-by-order the derivative of the generalisation error with respect to $\lambda$. At first order we have
\begin{equation}
    \pdv{{\Egen}_{1}}{\lambda_1} = \frac{2 \alpha  (\lambda_0-\din )}{\left(\alpha ^2+2 \alpha  (\lambda_0-1)+(\lambda_0+1)^2\right)^{3/2}}=0
\end{equation}
which vanishes as soon as $\lambda_0 = \din$ as computed previously.
Looking at the derivative at second order in $\epsilon$ allows to fix $\lambda_1$ as 
\begin{equation}
    \lambda_1 = \frac{\left(\alpha ^2+2 \alpha  (\lambda_0-1)+(\lambda_0+1)^2\right)
   \left(\beta ^2-2 \beta  (\lambda_0+1)-\din +\dout +2 \lambda_0+1\right)}{\alpha ^2+\alpha  (3 \din -\lambda_0-2)+(\lambda_0+1) (3 \din -2 \lambda_0+1)} \, .
\end{equation}

Thus the optimal order parameter can be written as $\lambdaopt = \din + \lambda_1 \epsilon + \mathcal{O}(\epsilon^2)$ and the value of the error is
\begin{equation}
\begin{aligned}
    \Egen^{\ell_2} &= {\Egen^{\mathrm{BO}}}_{0} + \bigg[ \frac{1}{2\kappa} \Big(2 \alpha ^2 (\beta -1)+\alpha  (\beta  (\beta +2 \din -6)-3 \din +\dout +5) \\
    & + (\din +1) \left(\beta ^2-\din +\dout - 1\right) \Big)  + \frac{1}{2} \left(-2 \alpha (\beta -1) + \beta ^2 + \din - \dout - 1\right) \bigg]\epsilon \\
    &+ \mathcal{O}(\epsilon^2)\,,
\end{aligned}
\end{equation}
with the definitions of eq.~(\ref{eq:order-zero-gen-error}) and
\begin{equation}
\begin{aligned}
    \kappa &= \sqrt{\alpha ^2+2 \alpha  (\din -1) + (\din +1)^2} \,.
\end{aligned}
\end{equation}

\section{Simulations details}

\subsection{Solving the self-consistent equations for ERM and BO}

The self-consistent equations from Proposition~\ref{thm:main-results} are written in a way amenable to be solved via fixed-point iteration. From a random initialisation, we iterate through both the hat and non-hat variable equations until the maximum absolute difference between the order parameters in two successive iterations falls below a tolerance of $10^{-9}$.

To speed-up convergence we use a damping scheme, updating each order parameter at iteration $i$, designated as $x_{i}$, using $x_{i} := x_i \mu + x_{i-1} (1-\mu)$, with $\mu$ as the damping parameter ($\mu \in [0.7, 0.9]$ was found most effective).

The Bayes-optimal (BO) fixed-point equations involve integrals which we could not express explicitly, necessitating a numerical integrator. The integration region must be subdivided into smaller regions to improve accuracy.
In the iteration of the self-consistent equations, the above-mentioned damping trick is employed.

Once convergence is achieved for fixed $\lambda$ and possibly $a$, hyper-parameters are optimised using the Nelder-Mead optimiser, a gradient-free numerical minimisation procedure.

\subsection{Numerical simulations for ERM}
\label{SM:numericalERM}

We obtain numerical simulations for the Empirical Risk Minimisation (ERM) by sampling synthetic datasets using eq.~(\ref{eq.datamodel}). The ERM problem is solved as follows:
\begin{itemize}[noitemsep,wide=0pt]
    \item for $\ell_2$ ERM is performed using the closed form solution. If we define the data matrix $\Phi \in \mathbb{R}^{n,d}$ as $\Phi_{ij} = x_i^j$, then the $\ell_2$ ERM problem is solved by:
    \begin{equation}\label{eq:ridge-inverse-estimator}
        \hat{\boldsymbol{w}} = \qty(\frac{\Phi^\top \Phi}{d} + \lambda \, \mathrm{Id})^{-1} \frac{\Phi^\top \boldsymbol{y}}{\sqrt{d}}
    \end{equation}
    \item For the Huber~loss, ERM is performed using the routine of the Scikit-learn library \cite{scikit-learn}, which uses a L-BFGS optimiser \cite{Liu1989}.
    \item For the $\ell_1$ loss we use its mapping onto Huber with small scale parameter, setting $a=10^{-3}$. To obtain the same result as the Huber regression with regularisation parameter $\lambda$ oen has to consider the $\ell_1$ regression with regularisation parameter $a \lambda$.
\end{itemize}

For the comparison with synthetic data in Figure~\ref{fig:figure-1} we first find the optimal parameters $(\lambda, a)$ for the generalisation error from the explicit solution of the fixed point equations. These values are then plugged into the ERM procedures described above.

To evaluate the estimation error it is pretty straightforward from the definition in eq.~(\ref{eq:estimation-error}). To evaluate the excess generalisation error from the definition of eq.~(\ref{eq:gen-error}), we generated $\ntest$ sample dataset from the same teacher and, given an estimate for the weights $\what$ we evaluated the following
\begin{equation}
    \frac{1}{\ntest} \sum_{i=1}^{\ntest}  \qty(y_i - \frac{\boldsymbol{x} \cdot \what}{\sqrt{d}})^2 - \frac{1}{\ntest} \sum_{i=1}^{\ntest}  \qty(y_i - \qty(1 - \epsilon + \beta \epsilon) \frac{\boldsymbol{x} \cdot \wstar}{\sqrt{d}})^2 \,,
\end{equation}
where the rescaling $(1-\epsilon + \epsilon \beta)$ is the one with minimal generalisation error.

\begin{figure}[t!]
    \centering
    \begin{tikzcd}
    {\fout}_{\mu}^{t} \arrow[rrrr, bend left] \arrow[r] & \gamma_{i}^{t} \arrow[r] \arrow[rd] & \hat{\mathrm{w}}^{t+1}_i \arrow[rr] &                                                 & \omega_{\mu}^{t+1} \arrow[rd] \arrow[r] & {\fout}_{\mu}^{t+1} \arrow[r] & \gamma_{i}^{t+1} \\
    \Lambda_{\mu}^{t} \arrow[rru] \arrow[rr] \arrow[ru]          &                                     & c^{t+1}_i \arrow[r]                 & V_{\mu}^{t+1} \arrow[ru] \arrow[rr] \arrow[rru] &                                         & \Lambda_{i}^{t+1} \arrow[ru]           &                 
    \end{tikzcd}
    \caption{Sketch of the dependence of the various quantities in GAMP algorithm. }
    \label{fig:dependance-g-amp}
\end{figure}
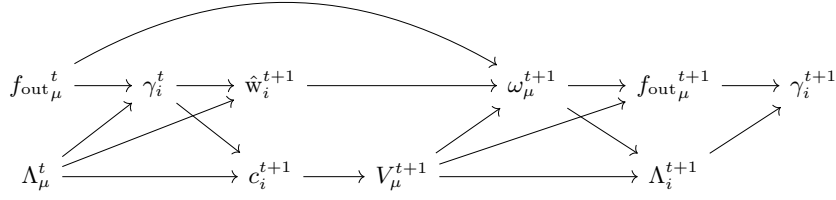

\subsection{Numerical simulations for BO: the Generalised Approximate Message Passing (GAMP) algorithm}
\label{SM:GAMP}

We summarise here the Generalised Approximate Message Passing (GAMP) algorithm, firstly introduced in \cite{Rangan}. We will follow the notation in \cite{aubin2020generalisation}. Given loss and  regularisation functions $\fout$ and $\fw$, computed as in Appendix \eqref{SM:thm-ERM}, we can estimate the student weight $\hat{\boldsymbol{w}}$ of the associated ERM problem by iterating the following fixed point system
\begin{equation}
\begin{aligned}[c]
    {\fout}_{\mu}^{t} &= \fout (y_\mu, \omega_\mu^t, V_\mu^t) \\
    \hat{\boldsymbol{w}}^{t+1}_i &= \fw(\gamma^t_i, \Lambda^t_i) \\
    c^{t+1}_i &= \partial_\gamma \fw(\gamma^t_i, \Lambda^t_i) \\
\end{aligned}
\qquad\qquad\qquad
\begin{aligned}[c]
    \Lambda_i^t &= - \frac{1}{d} \sum_{\mu = 1}^n X_{\mu i}^2 \partial_\omega {\fout}_{\mu}^{t} \\
    \gamma_i^t &= \frac{1}{\sqrt{d}} \sum_{\mu = 1}^n X_{\mu i} {\fout}_{\mu}^{t} + \Lambda_i^t \hat{\boldsymbol{w}}^t_i\\
    V_\mu^t &= \frac{1}{d} \sum_{i = 1}^d X_{\mu i}^2 \hat{c}^{t}_i\\
    \omega_\mu^t &= \frac{1}{\sqrt{d}} \sum_{\mu = 1}^n X_{\mu i} \hat{\boldsymbol{w}}^t_i - V^t_\mu {\fout}_{\mu}^{t-1}\\
\end{aligned}
\end{equation}

It has been shown \cite{Barbier2019} that if $\fw=\fwstar$ and $\fout = \foutstar$ then GAMP provides a Bayes Optimal estimation $\hat{\boldsymbol{w}} = \boldsymbol{w}_{\rm{BO}}$ as per  Section~\ref{sec:bo-and-erm}. We sketch in Figure~\ref{fig:dependance-g-amp} the dependence of the parameters of GAMP during one iteration step going from iteration time $t$ to iteration time $t+1$.

\section{Experiments with non-Gaussian data}\label{SM:real}
We provide here some numerical experiments for a model of data for which the samples $x$ are not Gaussian and non-i.i.d. In particular, we consider a regression problem on the MNIST dataset, preprocessed by extracting $300$ random Fourier data features. The experiments are presented in Figure \ref{fig:real_data}. We find results which are qualitatively compatible with what we presented in the main text. 
We observe (Figure \ref{fig:real_data}, left) that the estimation error is not consistent for all the three losses we consider. We also observe (Figure \ref{fig:real_data}, center and right) the sharp transition behaviour where Huber collapses to the $\ell_2$ loss, with $a_{\rm opt}$ diverging (the fact that the parameter does not diverge is to be expected, as for finite sample complexity $\alpha$, there will be a value $a = a_{\rm opt}$ above which the Huber loss behaves as the $\ell_2$ loss at all points in the training dataset, and increasing $a$ above such a threshold will not affect the learning problem anymore.). In particular, Figure \ref{fig:real_data}, center, shows the performance of the $\ell_2$ and Huber losses, and Figure \ref{fig:real_data}, right, the corresponding optimal value of the Huber scale parameter. We can see that the phase transition in $\Delta_{\rm OUT}$ is conserved even with real data.

The data is generated as follows. We select $n$ MNIST images, each considered as vector $X_i\in\mathbb{R}^{d_0}$ for $i=1, \dots, n$, and we compute for each of them $p$ random Fourier features by generating a random real-valued Gaussian matrix $\Omega \in\mathbb{R}^{p\times d_0}$, and computing $\hat X_i = \exp(-i\Omega X_i)\in\mathbb{C}^{p}$, where the exponential is taken component-wise. This provides $p$ complex features for each sample, which, when unpacked into their real and imaginary parts, translate into $2p$ random real Fourier features for each sample. We collect the features into the data matrix $\hat X\in\mathbb{R}^{2p \times n}$, whose aspect ratio $\alpha = n / 2p$ defines the sample complexity of the linear regression problem.
The labels are generated using the model of outliers considered in the manuscript \eqref{eq.datamodel} with a random $\boldsymbol{w}^\star \in \mathbb{R}^{2p}$.
The regression problem is solved, as detailed in Appendix \ref{SM:numericalERM}, for the three choices of losses considered in the submission. In this case, we fix $p = 300$ and choose $n$ depending on the value of $\alpha$.

The optimal values of the regularisation $\lambda$ and the scale parameter $a$ are computed by cross-validation on the estimation error.

\begin{figure}[h] 
    \vspace{0.5cm}
    \centering
    \includegraphics[width=0.31\textwidth]{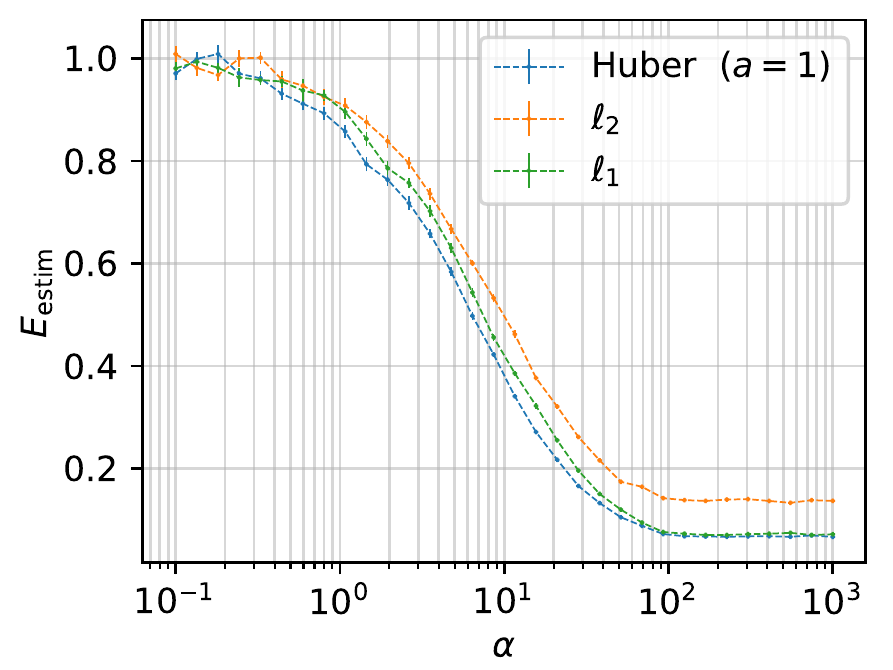}
    \hfill
    \includegraphics[width=0.31\textwidth]{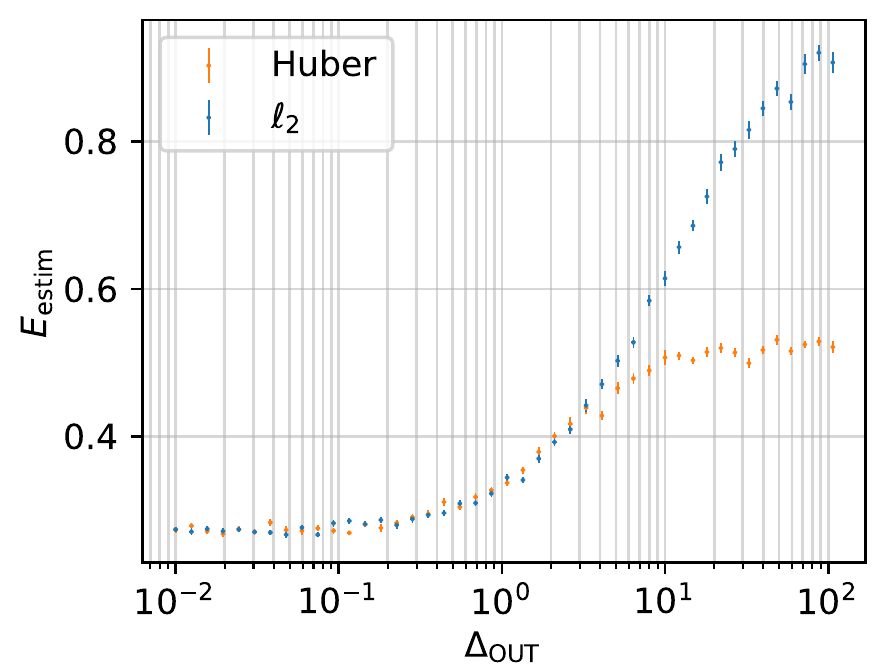}
    \hfill
    \includegraphics[width=0.31\textwidth]{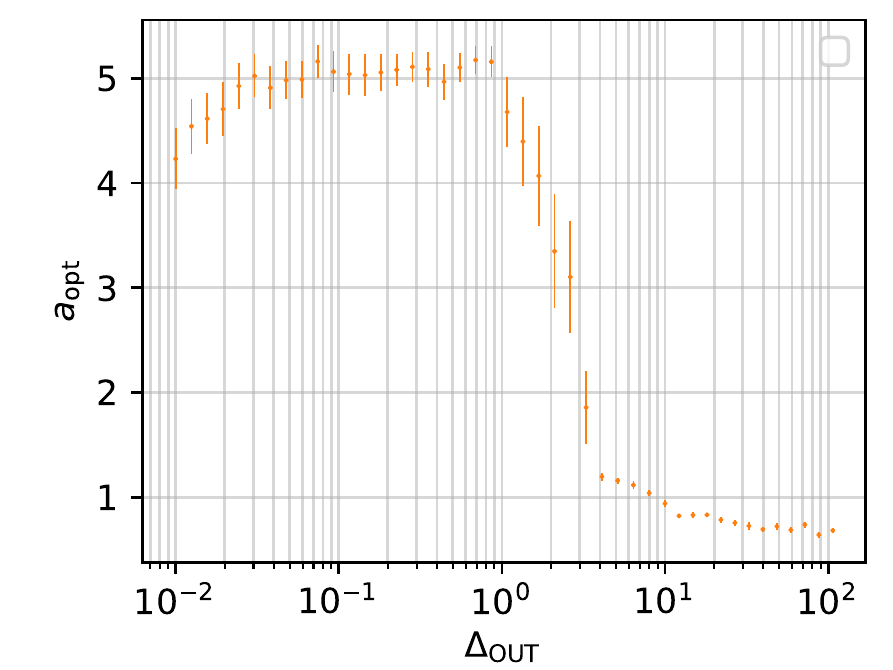}
    \caption{
    The points represent the mean and error of the mean of 16 instances of the simulations with data as described before. The simulations are performed with 30\% of outliers ($\epsilon = 0.3$), with an additive noise variance of the inliers $\Delta_{\rm IN} = 1$ and correlation coefficient of the outliers $\beta = 0$. The $\ell_2$ regularisation parameter is chosen to be the one that minimises the estimation error in all three Figures. \textbf{(Left)} This plot is the counterpart to Figure \ref{fig:figure-1} (right) of the main text. Here we vary the sample complexity $\alpha$ and fix the noise variance of the outliers $\Delta_{\rm OUT} = 5$. Additionally for the Huber loss we fixed the scale parameter $a = 1$. \textbf{(Center) \& (Right)} These two plots are the counterparts of Figure  \ref{fig:figure-2} (right) of the main text. In these plots we are varying $\Delta_{\rm OUT}$ with a fixed sample complexity $\alpha = 10$.
    }
    \label{fig:real_data}
\end{figure}

\section{Exploration of parameters' phenomenology}
\label{SM:exploration}
In this appendix we complement the figures in the main text by offering a more comprehensive exploration of the dependence of parameters $\alpha$, $\epsilon$, $\din$, $\dout$ and $\beta$.
In most of the figures we plot both the error (generalisation or estimation) in the upper panel and the value of the optimal hyper parameters in the bottom panel. Hyperparameters are plotted with colors matching the associated upper panel.

Gray lines correspond to the optimally tuned value of the scale parameter $a$ for Huber.

In Figure~\ref{fig:additional-sweep-alpha-params},~\ref{fig:additional-sweep-alpha-params-eps},~\ref{fig:additional-sweep-alpha-params-estimation-beta-small},~\ref{fig:additional-sweep-alpha-params-estimation-beta-large} and ~\ref{fig:angle-going-to-zero} we explore the dependence on $\alpha$:
\begin{itemize}[leftmargin=7mm]
    \item in Figures ~\ref{fig:additional-sweep-alpha-params} and \ref{fig:additional-sweep-alpha-params-eps} we plot the optinally-tuned excess generalisation error as a function of $\alpha$ for two choices of $\beta$. In figure \ref{fig:additional-sweep-alpha-params} we displayed the dependence on $\dout$. We notice that the phenomenology is qualitatively the same for the two different values of $\beta$ ($\beta = 0$ on the left column and $\beta = 0.2$ on the right column). As we expect from Theorem \ref{thm:large-alpha-ERM}. for large enough values of $\dout$ we achieve consistency with all losses, in particular with a diverging $\lambdaopt$. On the other hand if the value of $\dout$ is smaller than the threshold in \ref{thm:large-alpha-ERM} no choice of $\lambda\geq 0$ restores consistency in the large $\alpha$ limit.
    In Figure~\ref{fig:additional-sweep-alpha-params-eps} we show also that the same phenomenology holds when changing $\epsilon$. As in Figure ~\ref{fig:additional-sweep-alpha-params} we see that for certain choices of the parameters, at large enough $\epsilon$ the estimator is not consistent.

    \item In Figures~\ref{fig:additional-sweep-alpha-params-estimation-beta-small} and~\ref{fig:additional-sweep-alpha-params-estimation-beta-large}, we consider the optimally-tuned estimation error. In the figure \ref{fig:additional-sweep-alpha-params-estimation-beta-small} we focus on $\beta < 1$ and see that estimation cannot be found for positive values of $\lambda$. In figure \ref{fig:additional-sweep-alpha-params-estimation-beta-large} we have instead $\beta > 1$. Here consistent estimation can be achieved for an optimally tuned $\lambda$, but we stress that this requires a clean dataset. 

    \item In Figure~\ref{fig:angle-going-to-zero} we plotted the teacher-student angle as per eq.~(\ref{eq:angle-teacher-student-ovl}). We see that the figures are consistent with the claims of Section~\ref{SM:thm-large-alpha-angle} since even for non-optimised values of the parameters, the angle between the teacher and student vector is going to zero with $\alpha$.
\end{itemize}

Figure~\ref{fig:additional_sweeps-eps} and Figure~\ref{fig:additional_sweeps-delta-out} explore the dependency of the generalisation error on $\epsilon$ and $\dout$. We have qualitatively the same behaviour as in Figure~\ref{fig:figure-2} of the main text. Interestingly, the excess generalisation error is not zero in the limit $\epsilon\to 1$ for $\beta \neq 0$ for the parameters chosen in figures. The outlier regression problem in this limit is equivalent to regression without outlier, but with variance $\dout$.

Figure~\ref{fig:additional_phase-diag} shows quantitatively how the Huber loss performs better than the $\ell_2$ loss in the regression task at hand when optimally tuned. This complements Figure~\ref{fig:figure-3} of the main text, where we simply delimited the region where the two losses are identical. Interestingly, the advantage of Huber over $\ell_2$ grows the further away one is from the borders of this region.

\begin{figure}[b]
    \centering
    \includegraphics[width=0.39\columnwidth]{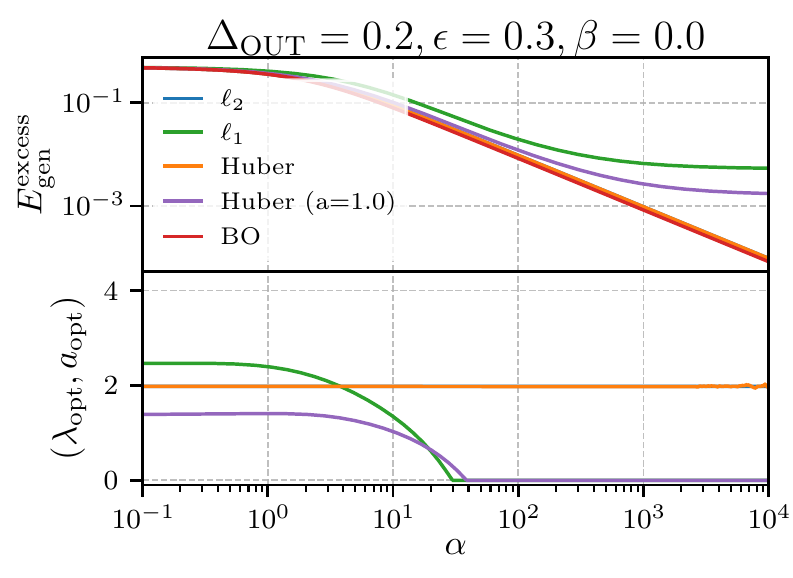}
    \includegraphics[width=0.39\columnwidth]{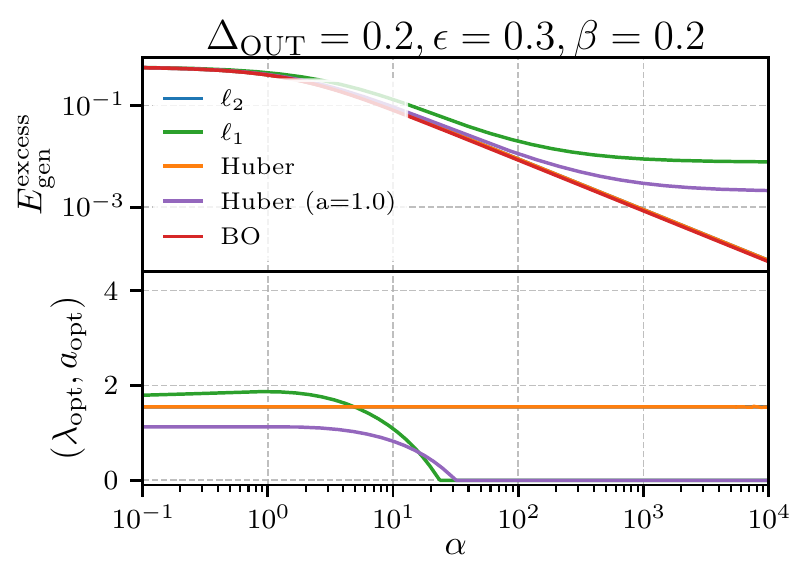} \\
    \includegraphics[width=0.39\columnwidth]{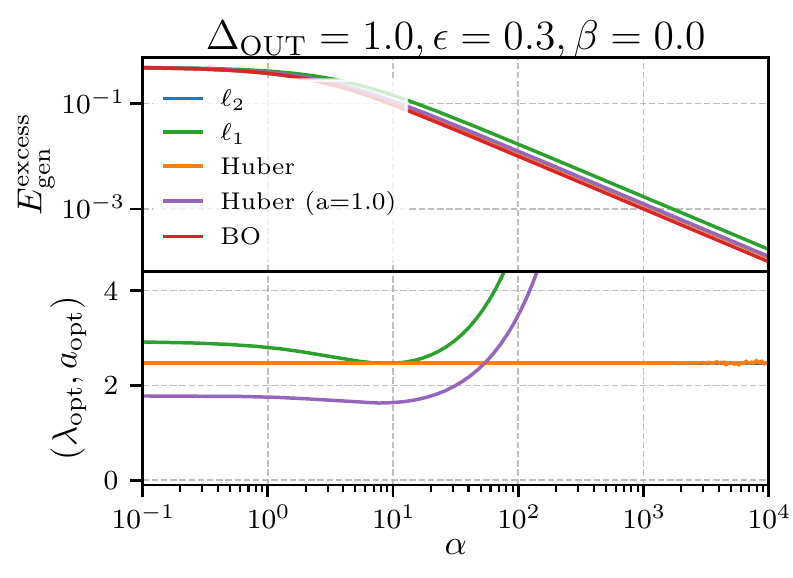}
    \includegraphics[width=0.39\columnwidth]{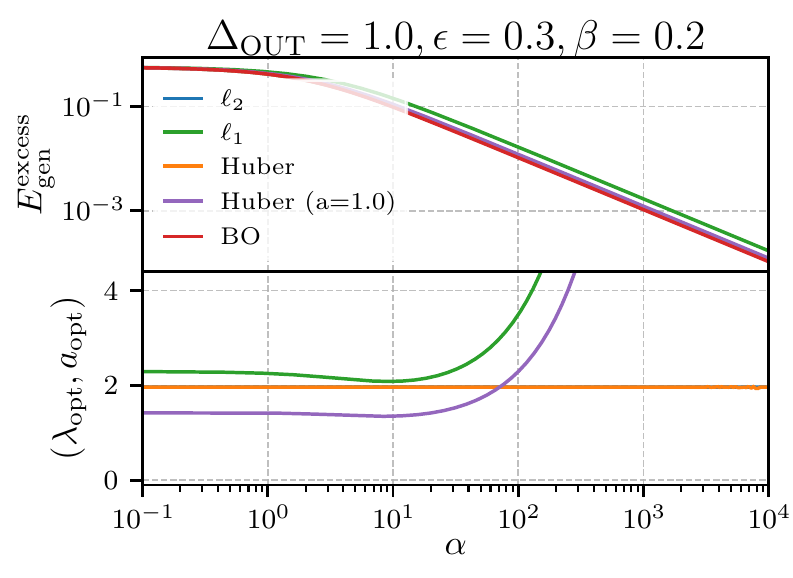} \\
    \includegraphics[width=0.39\columnwidth]{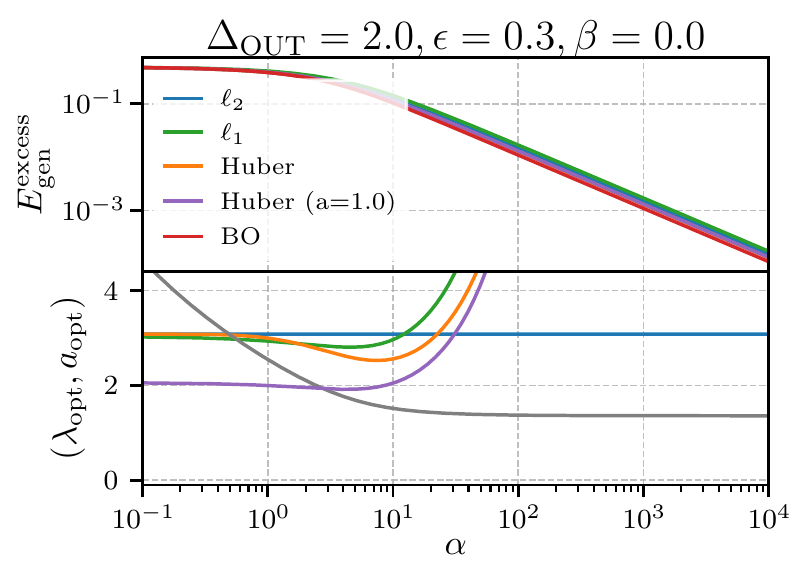}
    \includegraphics[width=0.39\columnwidth]{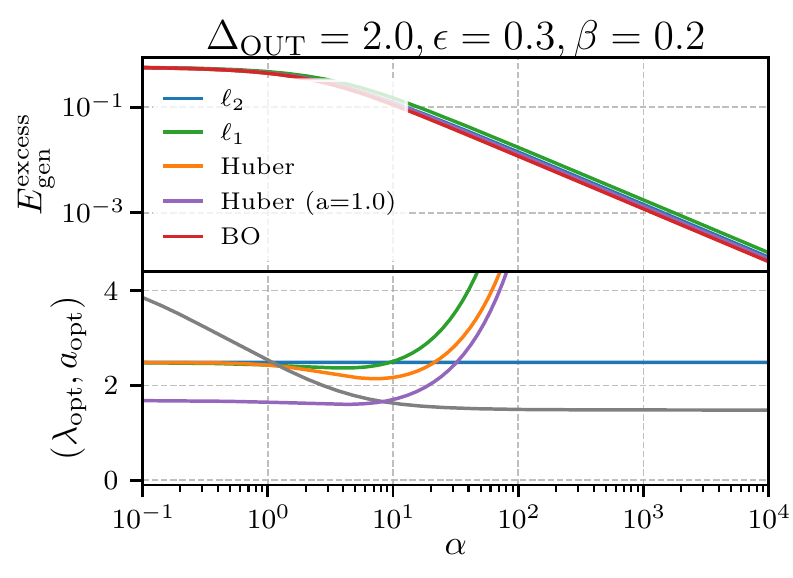} \\
    \includegraphics[width=0.39\columnwidth]{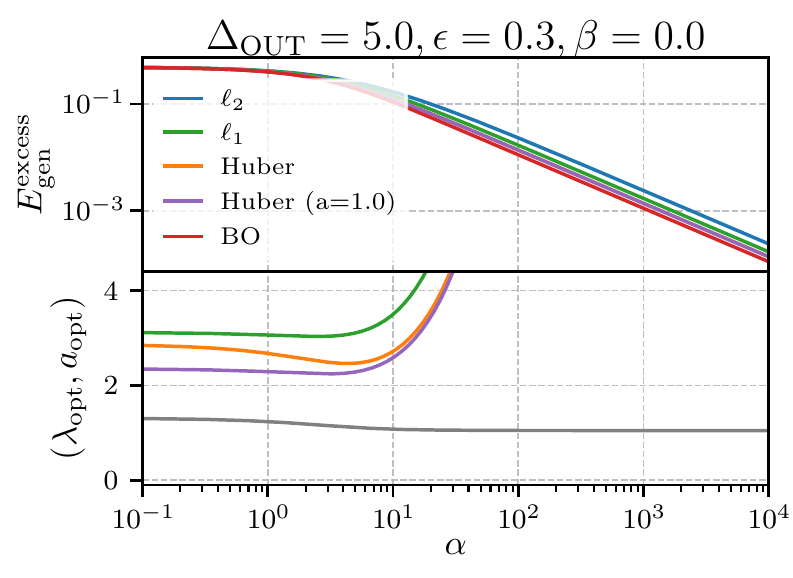}
    \includegraphics[width=0.39\columnwidth]{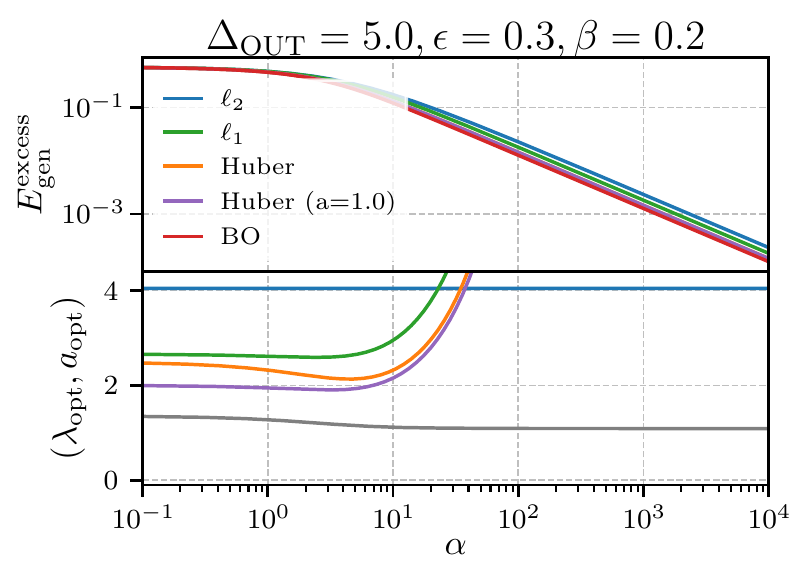}
    \caption{
        Excess generalisation error as a function of sample complexity for different value of the noise models' parameters. The hyper parameters, $\lambda$ and optionally $a$, for the ERM procedures are optimised to obtain the best generalisation error. 
    }
    \label{fig:additional-sweep-alpha-params}
\end{figure}

\begin{figure}[b]
    \centering
    \includegraphics[width=0.39\columnwidth]{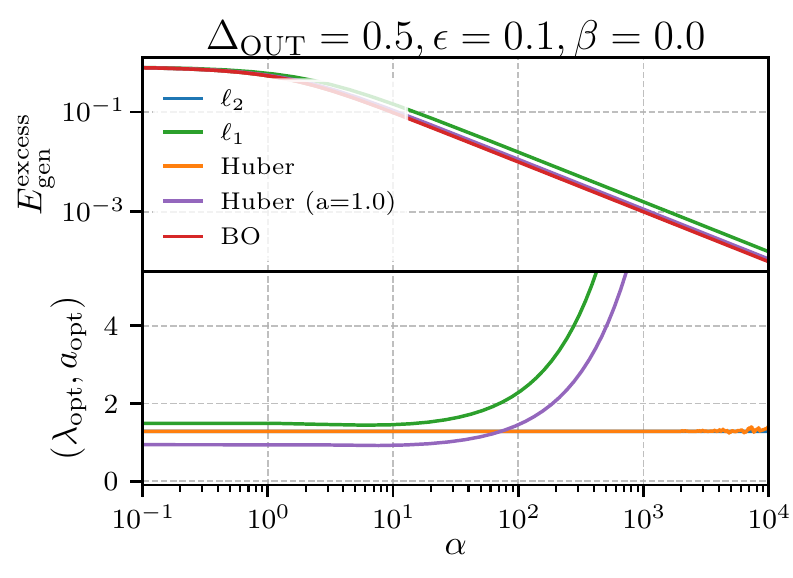}
    \includegraphics[width=0.39\columnwidth]{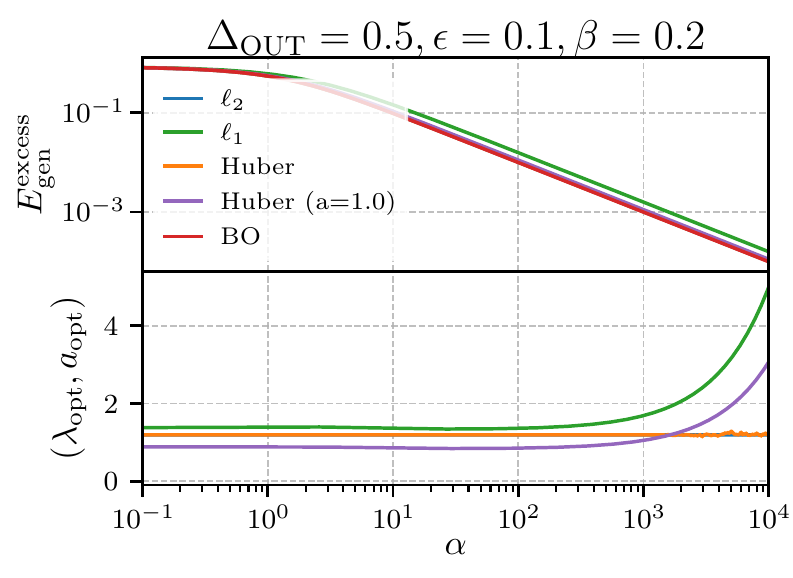} \\
    \includegraphics[width=0.39\columnwidth]{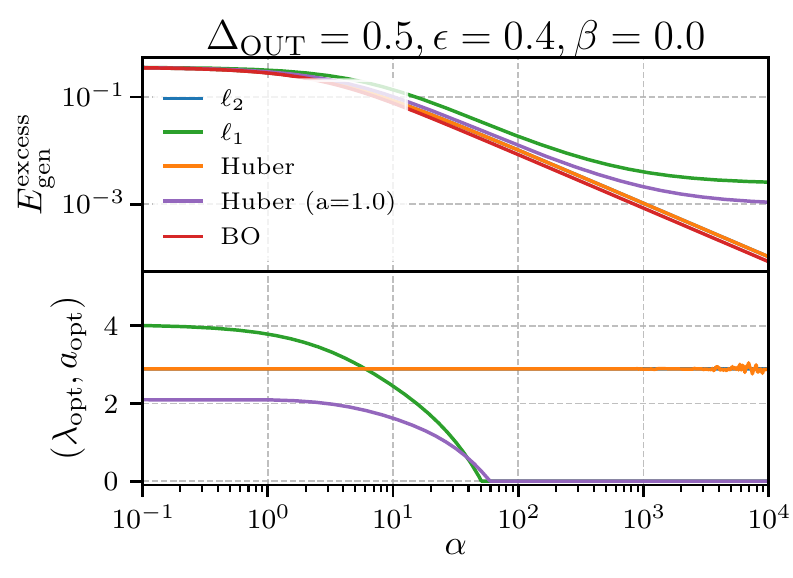}
    \includegraphics[width=0.39\columnwidth]{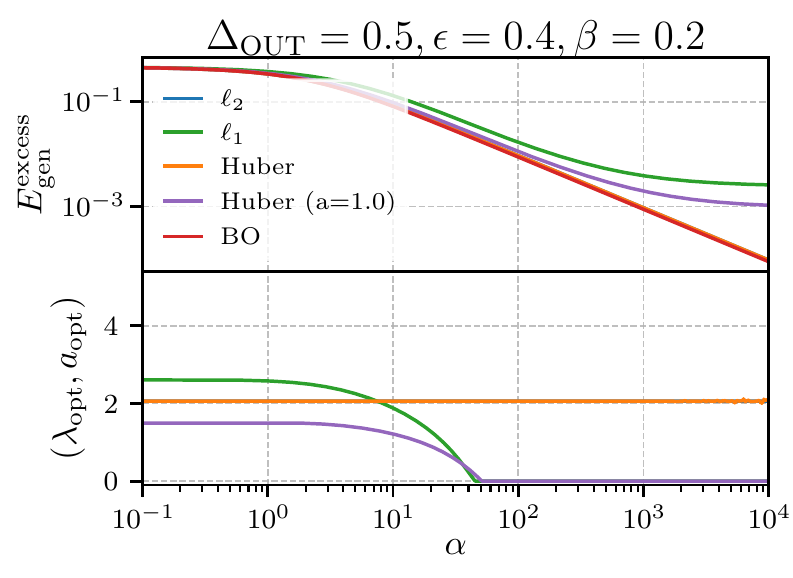} \\
    \includegraphics[width=0.39\columnwidth]{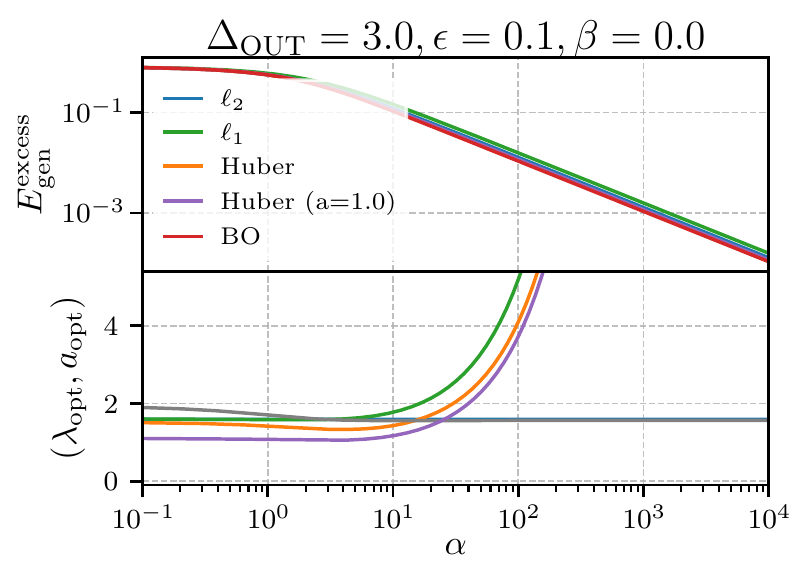}
    \includegraphics[width=0.39\columnwidth]{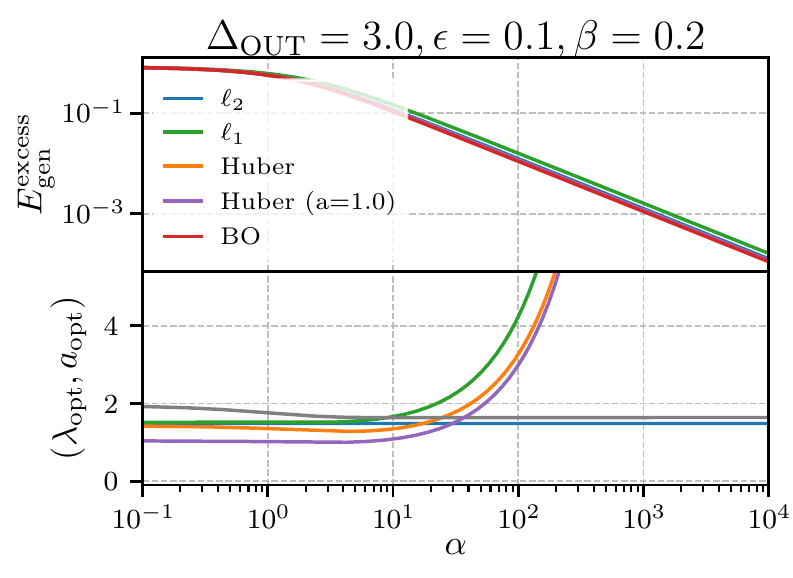} \\
    \includegraphics[width=0.39\columnwidth]{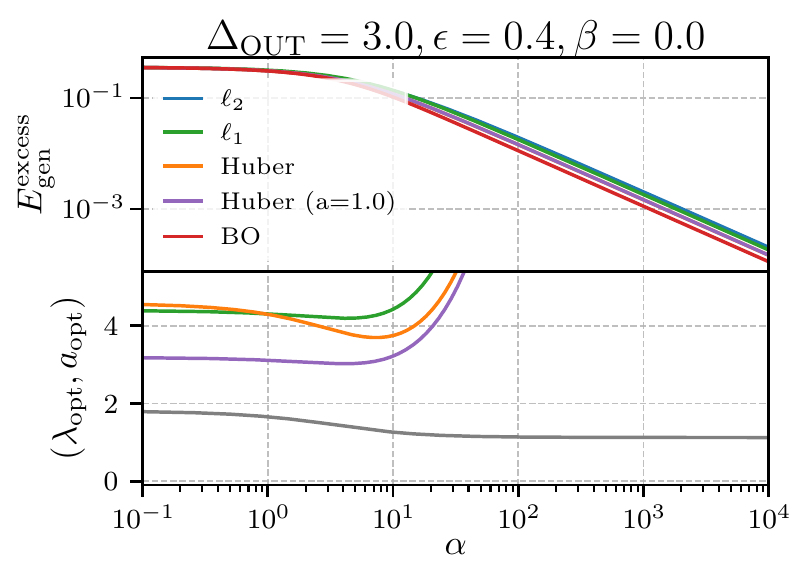}
    \includegraphics[width=0.39\columnwidth]{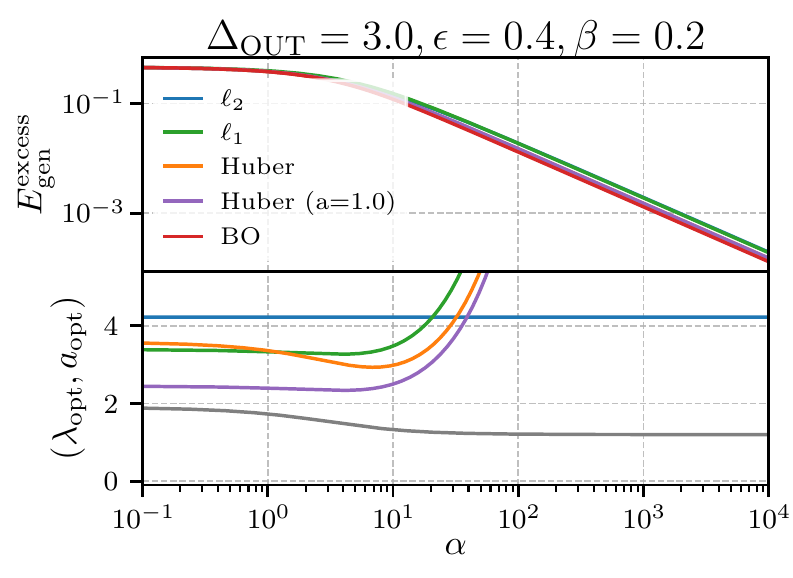}
    \caption{
        Excess generalisation error as a function of sample complexity for different value of the noise models' parameters. In all the plots shown we fixed $\din=1$. The hyper parameters, $\lambda$ and optionally $a$, for the ERM procedures are optimised to obtain the best generalisation error. 
    }
    \label{fig:additional-sweep-alpha-params-eps}
\end{figure}

\begin{figure}[b]
    \centering
    \includegraphics[width=0.39\columnwidth]{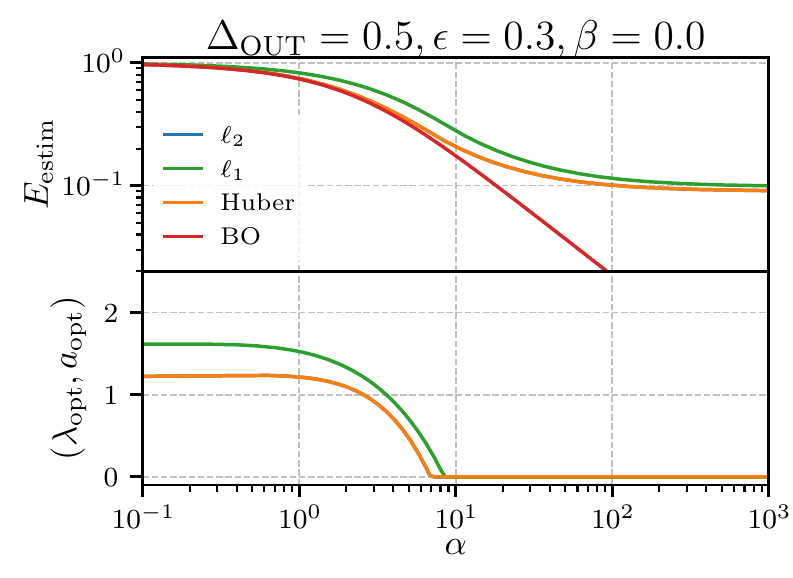}
    \includegraphics[width=0.39\columnwidth]{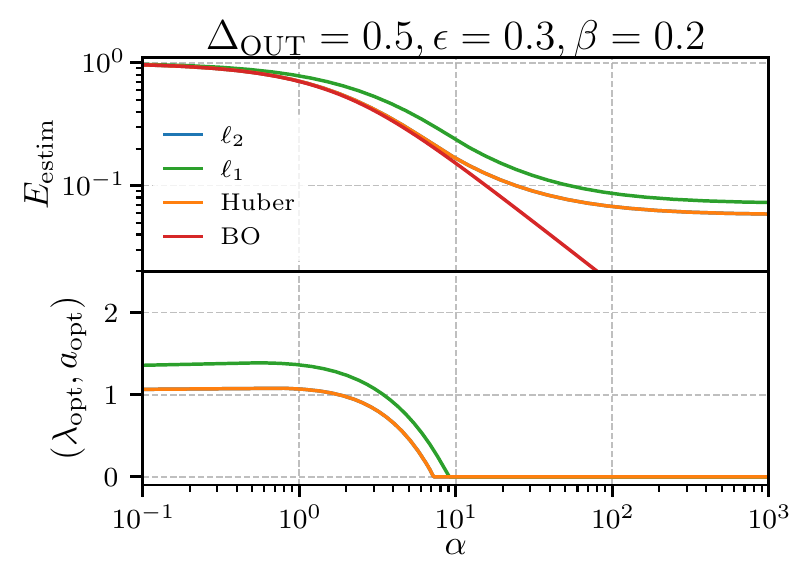} \\
    \includegraphics[width=0.39\columnwidth]{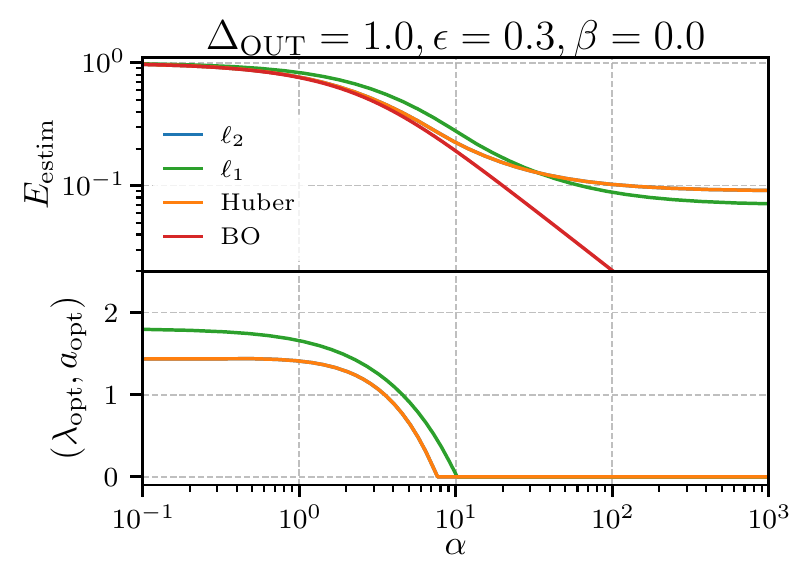}
    \includegraphics[width=0.39\columnwidth]{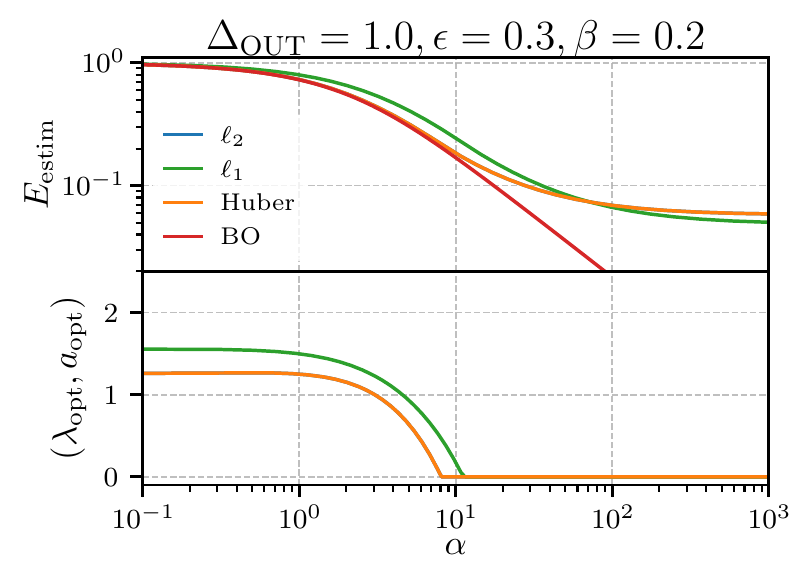} \\
    \includegraphics[width=0.39\columnwidth]{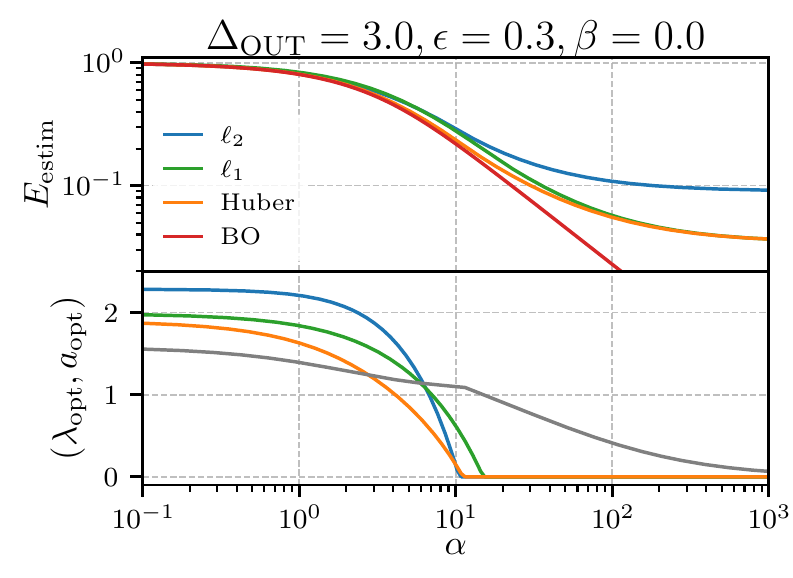}
    \includegraphics[width=0.39\columnwidth]{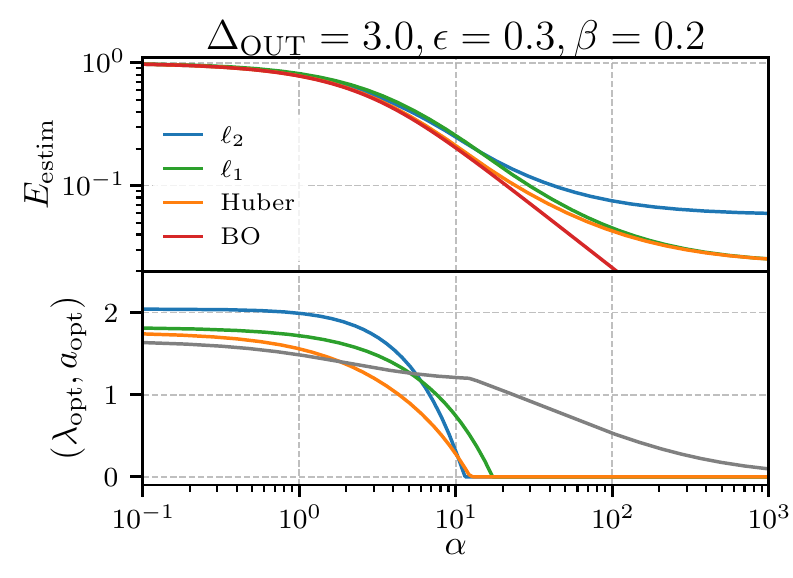}
    \caption{
        Estimation error as a function of sample complexity for different value of the noise models' parameters for $\beta < 1$. In all the plots shown we fixed $\din=1$. The hyper parameters, $\lambda$ and optionally $a$, for the ERM procedures are optimized to obtain the best generalisation error. 
    }
    \label{fig:additional-sweep-alpha-params-estimation-beta-small}
\end{figure}

\begin{figure}[b]
    \centering
    \includegraphics[width=0.39\columnwidth]{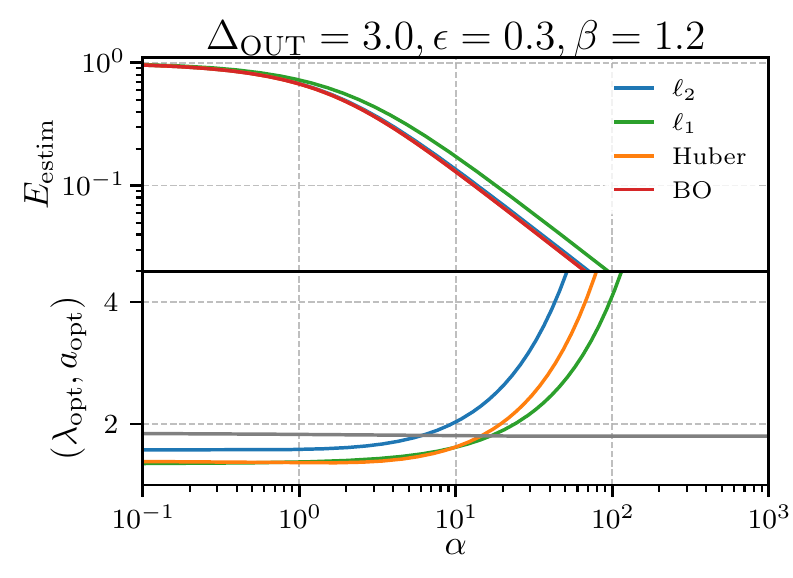}
    \includegraphics[width=0.39\columnwidth]{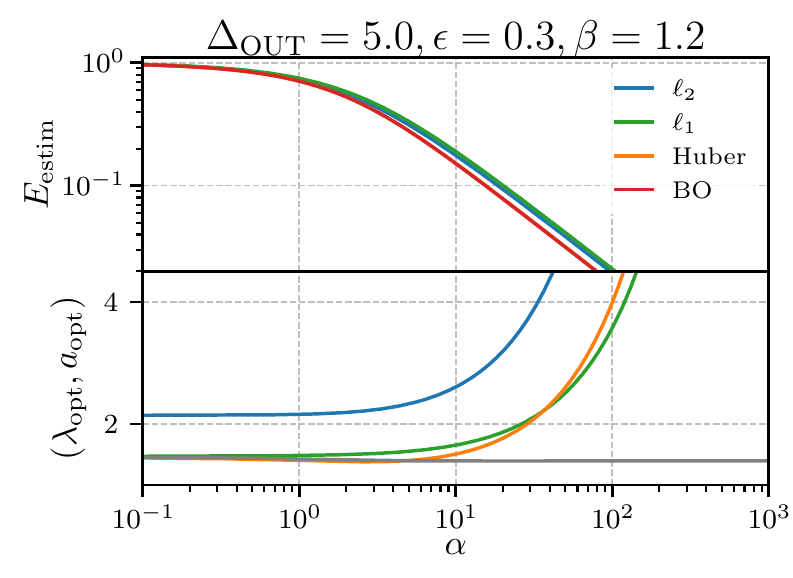}
    \caption{
        Estimation error as a function of sample complexity for different value of the noise models' parameters for $\beta > 1$. In all the plots shown we fixed $\din=1$. The hyper parameters, $\lambda$ and optionally $a$, for the ERM procedures are optimized to obtain the best generalisation error. 
    }
    \label{fig:additional-sweep-alpha-params-estimation-beta-large}
\end{figure}

\begin{figure}
    \centering
    \includegraphics[width=0.39\columnwidth]{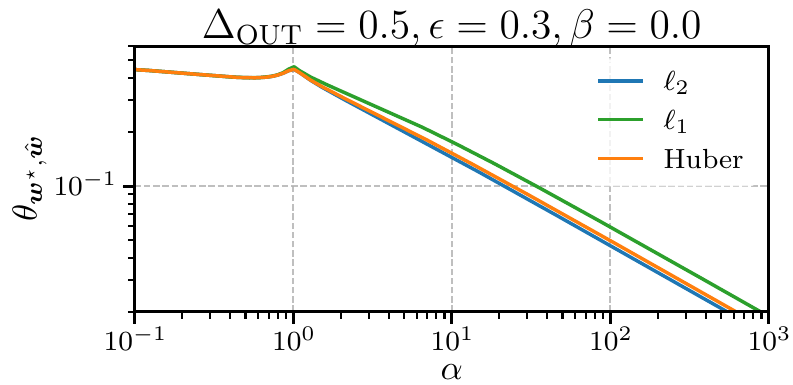}
    \includegraphics[width=0.39\columnwidth]{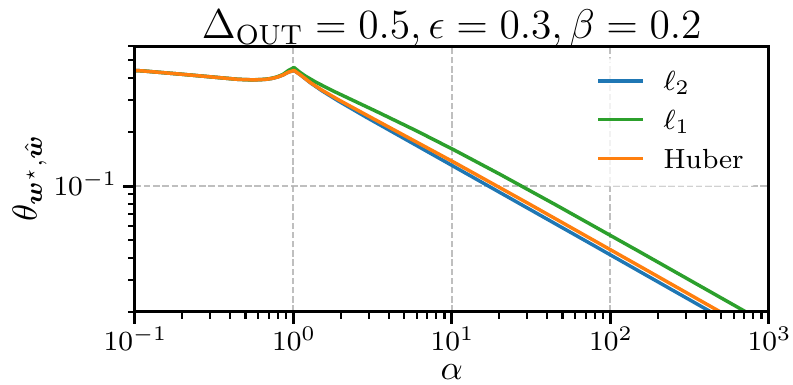} \\
    \includegraphics[width=0.39\columnwidth]{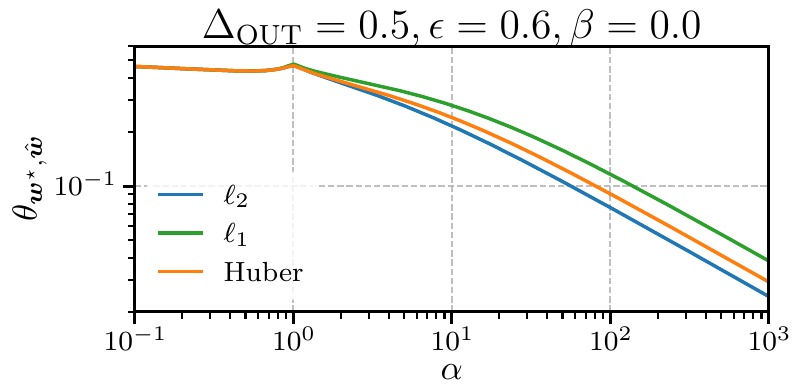}
    \includegraphics[width=0.39\columnwidth]{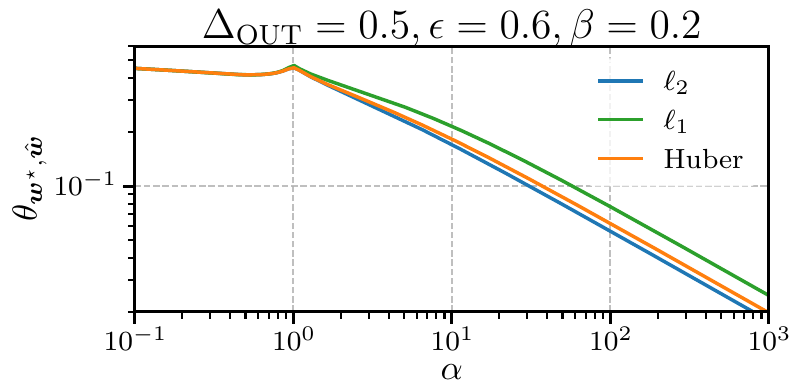} \\
    \includegraphics[width=0.39\columnwidth]{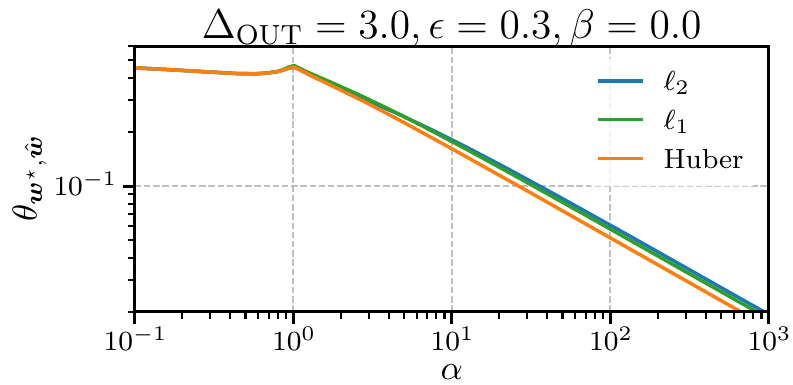}
    \includegraphics[width=0.39\columnwidth]{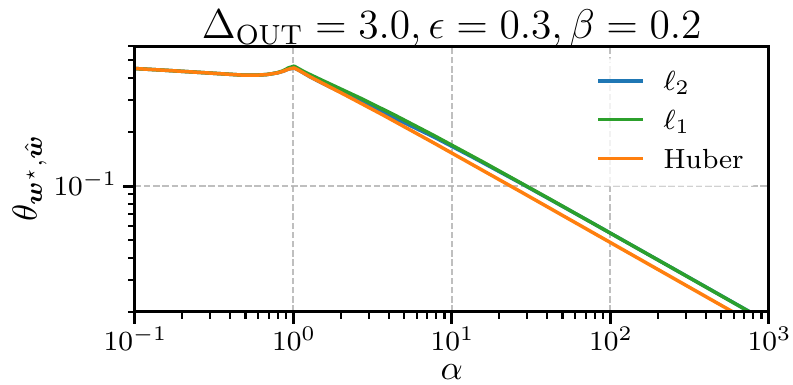} \\
    \includegraphics[width=0.39\columnwidth]{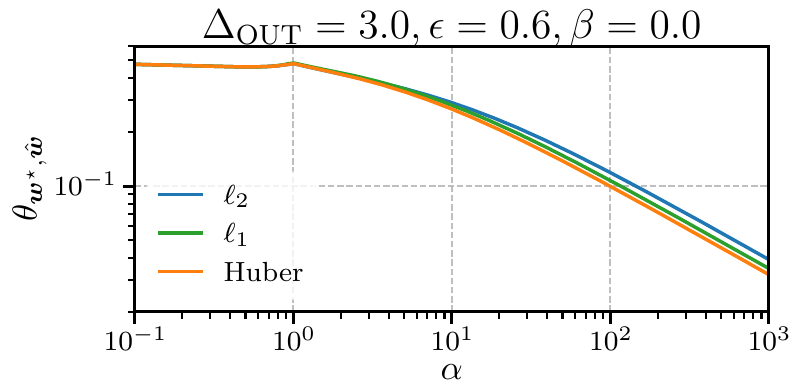}
    \includegraphics[width=0.39\columnwidth]{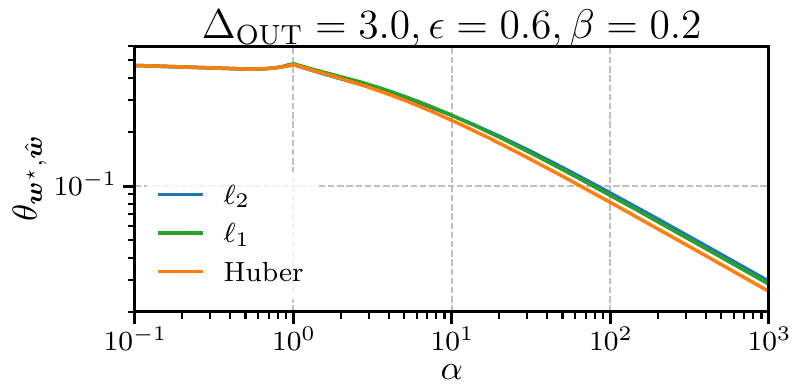}
    \caption{
        Angle between the teacher and the student --- eq.~(\ref{eq:angle-teacher-student-ovl}) --- as a function of the sample complexity $\alpha$. The figures are generated for a fixed value of regularisation parameter $\lambda = 10^{-3}$ and in the case of the Huber loss for $a = 1$. For each plot we fixed $\din = 1$ and the other parameters chosen are indicated on top of the plot.
    }
    \label{fig:angle-going-to-zero}
\end{figure}

\begin{figure}
    \centering
    \includegraphics[width=0.39\columnwidth]{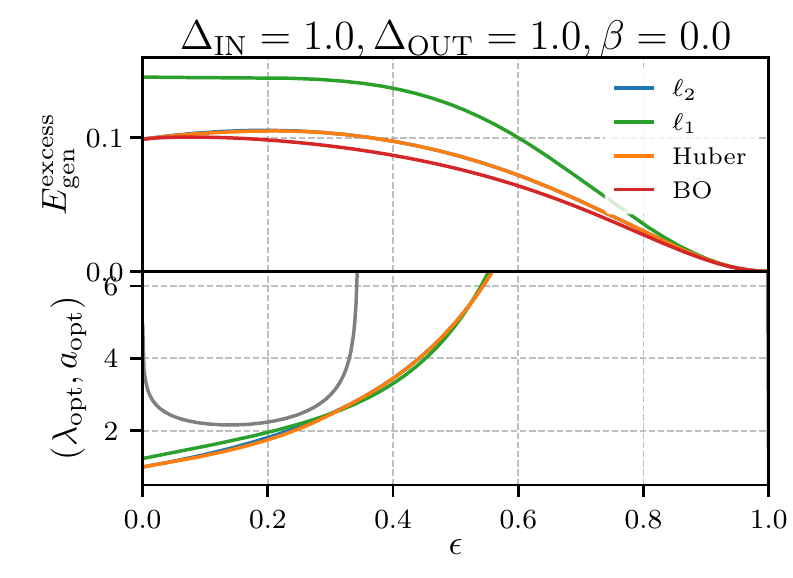}
    \includegraphics[width=0.39\columnwidth]{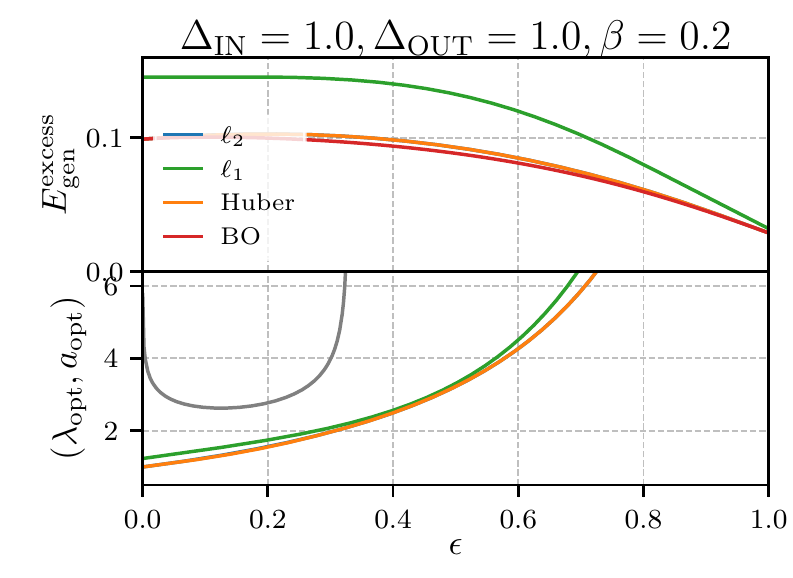} \\ 
    \includegraphics[width=0.39\columnwidth]{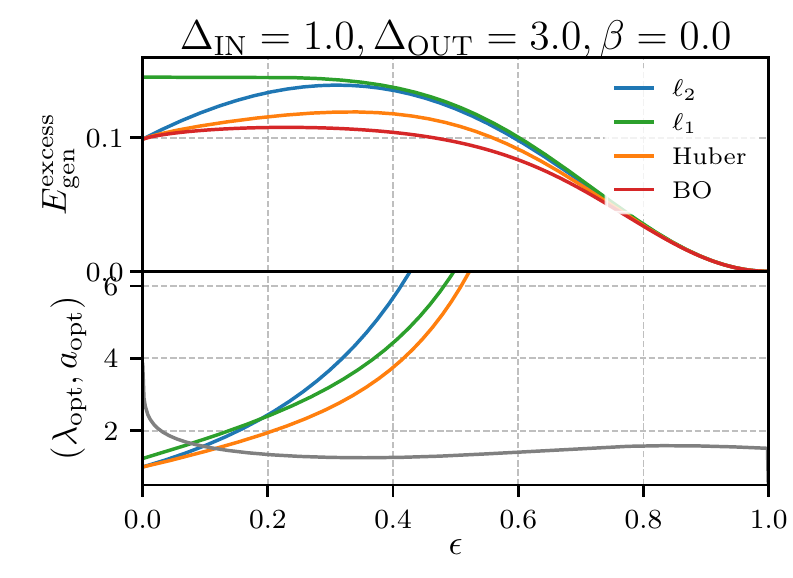}
    \includegraphics[width=0.39\columnwidth]{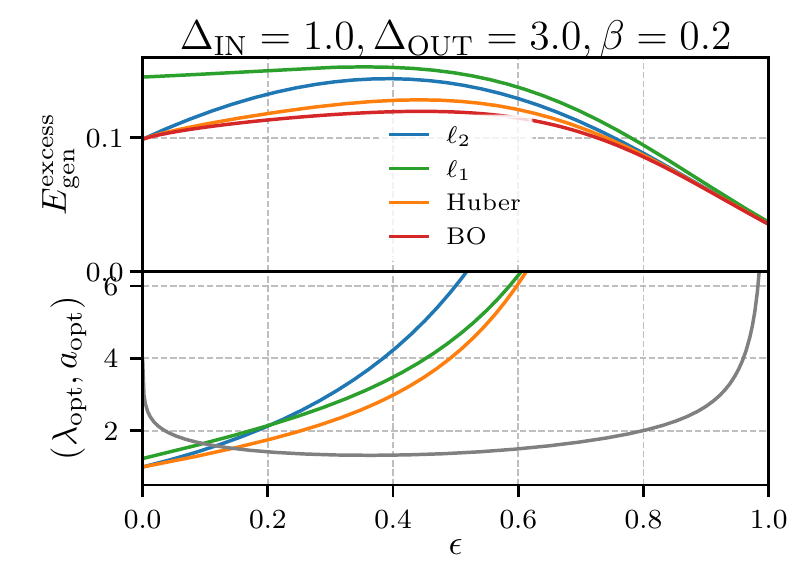}
    \caption{
        Excess generalisation error as a function of $\epsilon$. In all the plots shown we fixed $\alpha = 10$. The hyper parameters, $\lambda$ and optionally $a$, for the ERM procedures are optimized to obtain the best generalisation error. 
    }
    \label{fig:additional_sweeps-eps}
\end{figure}

\begin{figure}
    \centering
    \includegraphics[width=0.39\columnwidth]{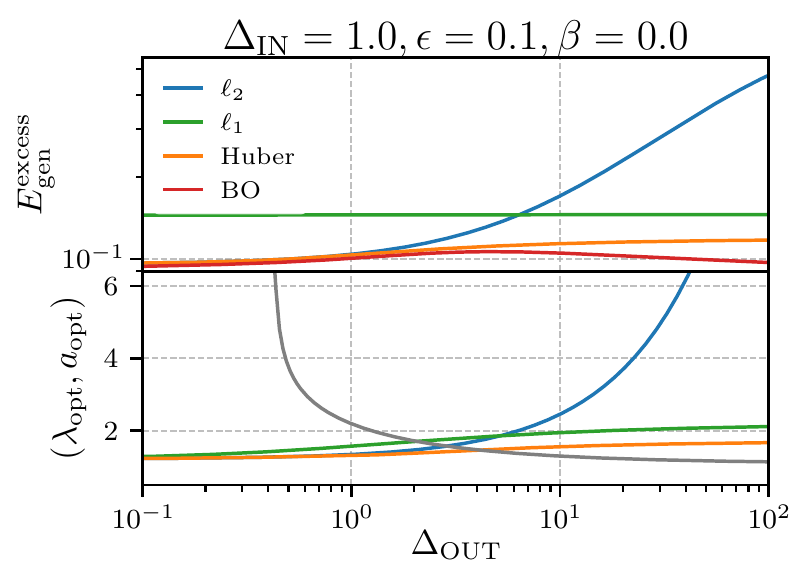}
    \includegraphics[width=0.39\columnwidth]{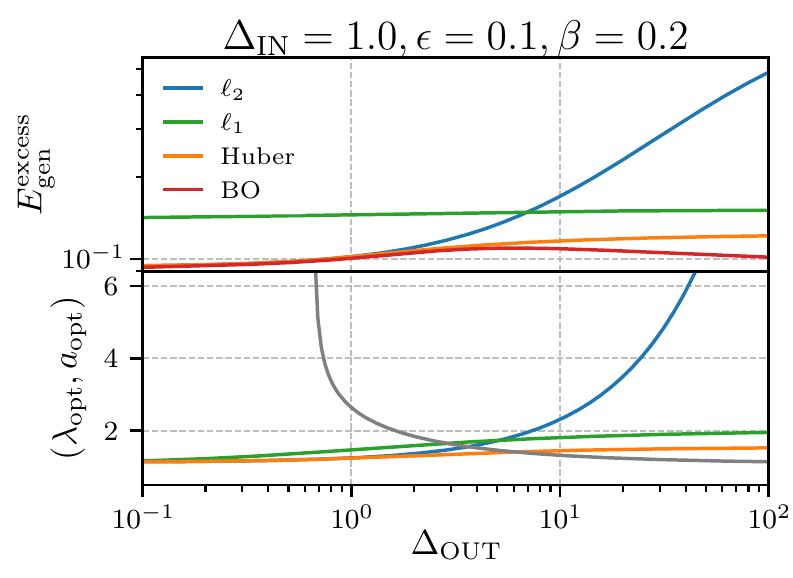} \\
    \includegraphics[width=0.39\columnwidth]{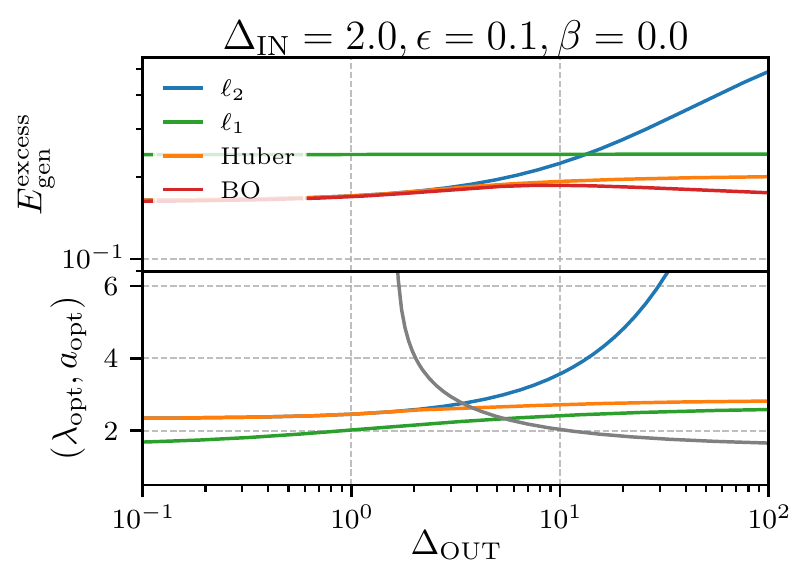}
    \includegraphics[width=0.39\columnwidth]{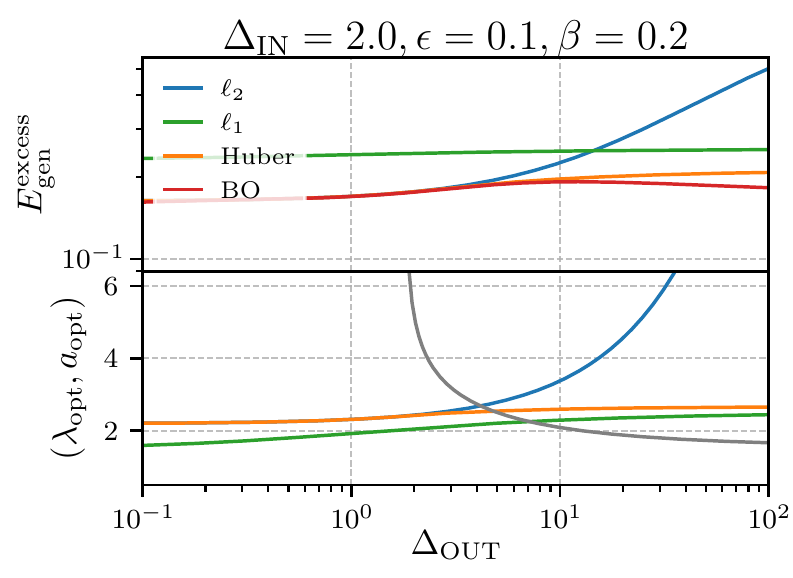}
    \caption{
        Excess generalisation error as a function of $\dout$. In all the plots shown we fixed $\alpha = 10$. The hyper parameters, $\lambda$ and optionally $a$, for the ERM procedures are optimized to obtain the best generalisation error. 
    }
    \label{fig:additional_sweeps-delta-out}
\end{figure}

\begin{figure}[hb]
    \centering
    \includegraphics[width=0.39\columnwidth, height=0.34\textwidth]{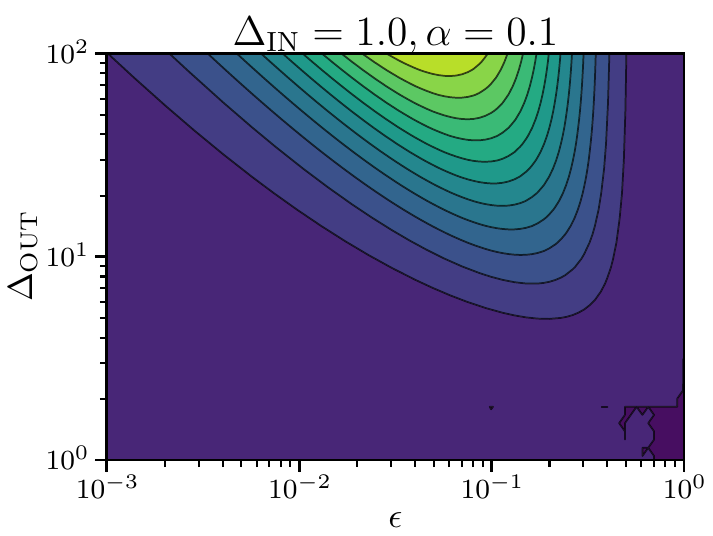}
    \includegraphics[width=0.1\columnwidth, height=0.34\textwidth]{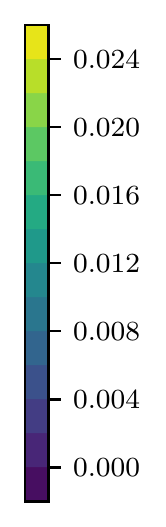} \hfill
    \includegraphics[width=0.39\columnwidth, height=0.34\textwidth]{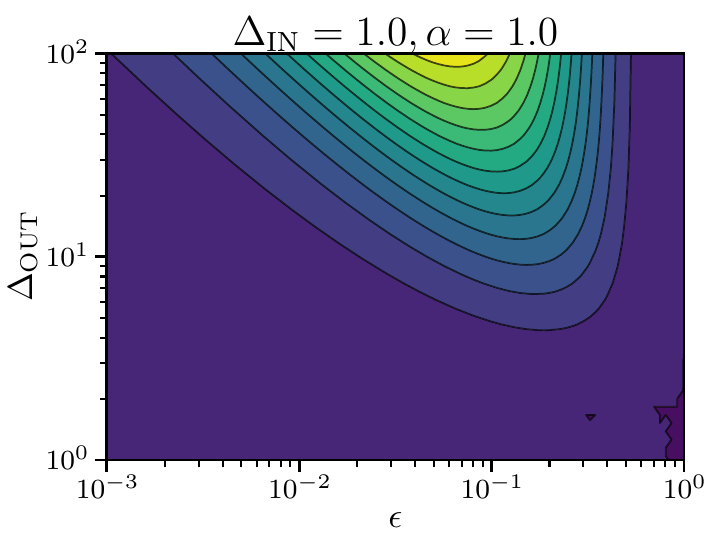}
    \includegraphics[width=0.1\columnwidth, height=0.34\textwidth]{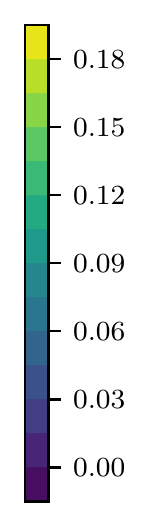} \\
    \includegraphics[width=0.39\columnwidth, height=0.34\textwidth]{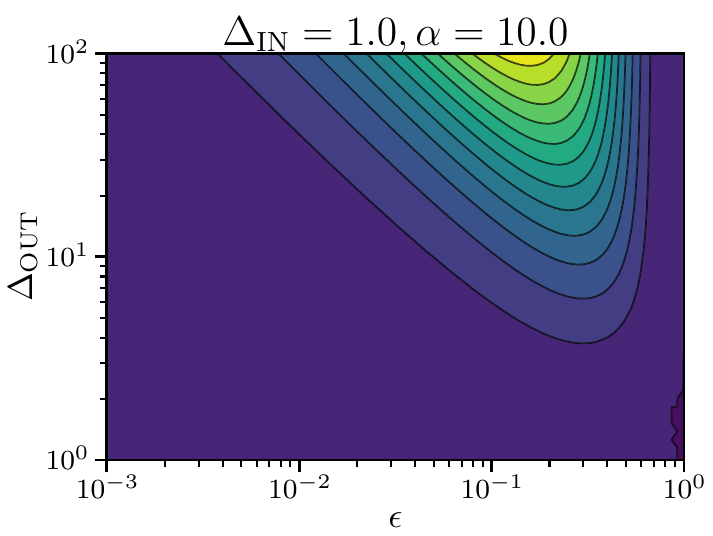}
    \includegraphics[width=0.1\columnwidth, height=0.34\textwidth]{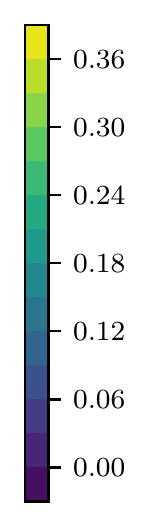}\hfill
    \includegraphics[width=0.39\columnwidth, height=0.34\textwidth]{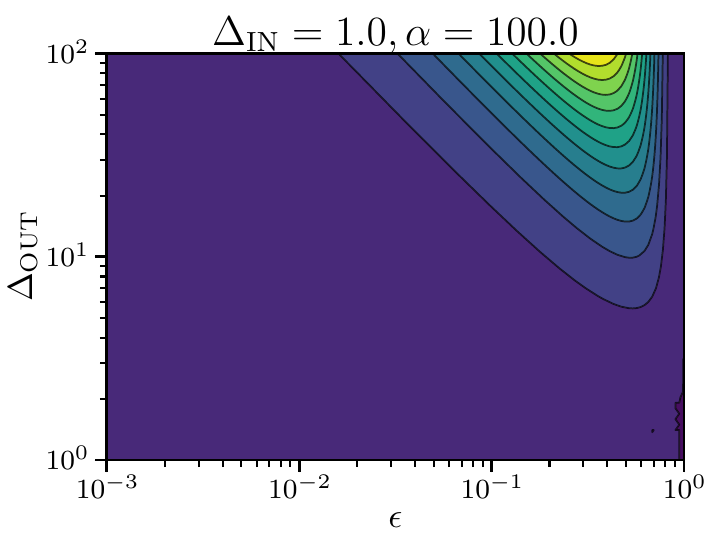}
    \includegraphics[width=0.1\columnwidth, height=0.34\textwidth]{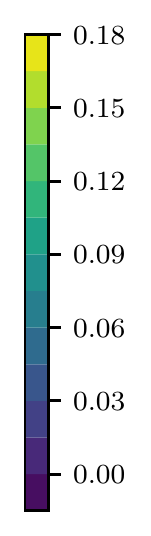}
    \caption{
        Difference between the Huber and $\ell_2$ generalisation error as a function of $\dout$ and $\epsilon$. Here we fix $\din=1$. Hyper-parameters for both losses are optimally tuned. In the dark blue region the difference is identically zero.
    }
    \label{fig:additional_phase-diag}
\end{figure}

\end{document}